\documentclass{article}
\usepackage{graphicx}
\usepackage{color}
\usepackage{amscd}
\usepackage{framed}
\usepackage{bbm}
\usepackage{amsmath,amsthm,amssymb}
\usepackage{fancyhdr,a4wide}
\newcommand{\mc}{\mathcal}
\newcommand{\mbb}{\mathbb}
\newcommand{\mr}{\mathrm}
\newcommand{\argmin}{\mathop{\rm argmin}\limits}

\newcommand{\veca}{\mathbf{a}}

\newcommand{\vecu}{\mathbf{u}}
\newcommand{\vecv}{\mathbf{v}}
\newcommand{\vecw}{\mathbf{w}}
\newcommand{\vecx}{\mathbf{x}}
\newcommand{\vecX}{\mathbf{X}}

\newcommand{\vecz}{\mathbf{z}}

\newcommand{\vecbeta}{\boldsymbol \beta}
\newcommand{\vectheta}{\boldsymbol \theta}

\newcommand{\vecvarrho}{\boldsymbol \varrho}

\usepackage{bm}

\usepackage{algorithm}
\usepackage{algorithmic}
\renewcommand{\algorithmicrequire}{\textbf{Input:}}
\renewcommand{\algorithmicensure}{\textbf{Output:}}

\numberwithin{equation}{section}
\newtheorem{theorem}{Theorem}[section]
\newtheorem{proposition}{Proposition}[section]
\newtheorem{remark}{Remark}[section]

\newtheorem{definition}{Definition}[section]
\newtheorem{assumption}{Assumption}[section]

\begin{document}
\title{Robust and Sparse Estimation of Linear Regression Coefficients with  Heavy-tailed  Noises and Covariates}
\author{Takeyuki Sasai
\thanks{Department of Statistical Science, The Graduate University for Advanced Studies, SOKENDAI, Tokyo, Japan. Email: sasai@ism.ac.jp}}
\maketitle
\begin{abstract}
	Robust and sparse estimation of linear regression coefficients is investigated.
	The situation addressed by the present paper is that covariates and noises are sampled from heavy-tailed distributions, and the covariates and noises are contaminated by malicious outliers.
	Our estimator can be computed efficiently. Further, the error bound of the estimator is nearly optimal.
\end{abstract}

\section{Introduction}
Sparse estimation has been studied extensively over the past 20 years to handle modern high-dimensional data e.g., \cite{Tib1996Regression, FanLi2001Variable,ZouHas2005Regularization,YuaLin2006Model,Don2006Compressed,CanTao2007Dantzig,RasWaiYu2010Restricted,Zha2010Nearly,SuCan2016Slope,BelLecTsy2018Slope,KolLouTsy2011Nuclear,NegWai2011Estimation,RohTsy2011Estimation,NegWai2012Restricted,CaiZha2013Sparse,Klo2014Noisy,KloLouTsy2017Robust,FanWanZhu2021Shrinkage}.
Because the advancement of computer technology has made it possible to collect very high dimensional data efficiently, sparse estimation will continue to be an important and effective method for high dimensional data analysis in the future.
On the other hand, in recent years, robust estimation methods for outliers or heavy-tailed distribution  have been developed rapidly e.g., \cite{NguTra2012Robust,CheCarMan2013Robust,LaiRaoVem2016Agnostic,BalDuLiSin2017Computationally,CheGaoRen2018robust,DiaKanKanLiMoiSte2019Robust,DiaKanKanLiMoiSte2017Being,KotSteSte2018Robust,DiaKamKanLiMoiSte2018Robustly,DalTho2019Outlier,CheDiaGe2019High,DiaGouTza2019Distribution,CheYesFlaBar2019Fast,CheDiaGeWoo2019Faster,DonHopLi2019Quantum,DiaKanKarPriSte2019Outlier,DiaKonSte2019Efficient,Tho2020Outlier,LiuSheLiCar2020High,Gao2020Robust,Chi2020Erm,MonGoeDiaSre2020Efficiently, DiaKanMan2020Complexity,DiaKonTzaZaf2020Learning,CheAraTriJorFlaBar2020Optimal,LeiLuhVenZha2020Fast,Hop2020Mean,PraBalRav2020Robust,LugMen2021Robust,BakPra2021Robust,DiaKanKonTzaZaf2021Efficiently,DepLec2022Robust,DalMin2022All}. These studies dealt with estimating problems of mean, covariance,  linear regression coefficients, half-spaces, parameters of Gaussian mixture moles, and so on. They are mainly interested in deriving sharp error bounds, deriving information-theoretical lower bounds of error bounds, and reducing computational complexity.

In the present paper, we consider sparse estimation of linear regression coefficients when covariates and noises are sampled from heavy-tailed distributions, and the samples are contaminated by malicious outliers. Define a normal sparse linear regression model as follows:
\begin{align}
	\label{model:normal}
	y_i = \vecx_i^\top\vecbeta^*+\xi_i,\quad  i=1,\cdots,n,
\end{align}
where $\left\{\vecx_i\right\}_{i=1}^n$ is a sequence of  independent and identically distributed (i.i.d.) random vectors, $\vecbeta^* \in \mbb{R}^d$ is the true coefficient vector, and $\left\{\xi_i\right\}_{i=1}^n$ is a sequence of i.i.d. random variables.
We assume the number of non-zero elements of $\vecbeta^*$ is $s\,(\leq d)$.
When an adversary injects outliers to the normal sparse linear regression model, \eqref{model:normal} changes as follows:
\begin{align}
	\label{model:adv}
	y_i = \vecX_i^\top\vecbeta^*+\xi_i+\sqrt{n}\theta_i,\quad  i=1,\cdots,n,
\end{align}
where $\vecX_i = \vecx_i+\vecvarrho_i$ for $i=1,\cdots,n$ and $\{\vecvarrho_i\}_{i=1}^n$ and $\{\theta_i\}_{i=1}^n$ are the outliers.
We allow the adversary to inject arbitrary values into arbitral $o$ samples of $\{y_i,\vecx_i\}_{i=1}^n$. Let $\mc{O}$ be the index set of the injected samples and $\mc{I} = (1,\cdots,n)\setminus \mc{O}$. Therefore, $\vecvarrho_i= (0,\cdots,0)^\top$ and $\theta_i=0$ hold for $i\in\mc{I}$. We note  that  $\{\vecvarrho_i\}_{i\in \mc{O}}$ and $\{\theta_i\}_{i\in \mc{O}}$ can be arbitral values and they are allowed to correlate freely among them and correlate with $\{\vecx_i\}_{i=1}^n$ and $\{\xi_i\}_{i=1}^n$.
The difficulty is not only that  $\{\vecvarrho_i\}_{i\in \mc{O}}$ and $\{\theta_i\}_{i\in \mc{O}}$ can take arbitral values but also that $\{\vecx_i\}_{i\in\mc{I}}$ and $\{\xi_i\}_{i\in\mc{I}}$ no longer follow sequences of i.i.d. random variables because we allow the adversary to freely select samples for injection. This kind of contamination by outliers is sometimes called strong contamination in contrast to the Huber contamination \cite{DiaKan2019Recent}. We note that the Huber contamination is more manageable because its outliers are not correlated to the inliers and do not restrict their independence.

Various studies \cite{KotSteSte2018Robust,PraSugSaiSicPra2020Robust, BakPra2021Robust,PenJogLoh2020robust, CheAraTriJorFlaBar2020Optimal} dealt with estimation of linear regression coefficients with samples drawn from some heavy-tailed distributions under the existence of outliers.
Some \cite{BalDuLiSin2017Computationally,Gao2020Robust,Chi2020Erm,Tho2020Outlier,NguTra2012Robust,CheCarMan2013Robust,DalTho2019Outlier, Gao2020Robust}  considered sparse estimation of linear regression coefficients when samples are drawn from Gaussian or subGaussian distributions under the existence of outliers. 
However, to the best of our knowledge, none of the findings on robust estimation of linear regression coefficients provides the result when the covariates and noises are drawn from heavy-tailed distributions and contaminated by outliers, and where the true coefficient vector is sparse.

Our result is as follows:	For the precise statement, see Theorem \ref{t:main} in Section \ref{sec:results}.
For any vector $\vecv$, define the $\ell_2$ norm of $\vecv$ as $\|\vecv\|_2$, and define $\vecx_{i_j}$ as the $j$-th element of $\vecx_i$.
Define $o = |\mc{O}|$, where $|\mc{S}|$ for a set $\mc{S}$ is the number of the elements of $\mc{S}$.
\begin{theorem}
	\label{theoreminformal1}
	Suppose that $\left\{\vecx_i \right\}_{i=1}^n$ is a sequence of i.i.d. random vectors with zero mean and with finite kurtosis.
	Suppose that, for any $1\leq j_1,j_2,j_3,j_4\leq d$, $\mbb{E}(x_{ij_1}x_{ij_2}x_{ij_3}x_{ij_4})^2$ exists.
	Suppose that  $\left\{\xi_i\right\}_{i=1}^n$ is a sequence of i.i.d. random variables whose absolute moment is bounded, and  that $\left\{\xi_i\right\}_{i=1}^n$ and $\left\{\vecx_i \right\}_{i=1}^n$ are independent.
	Then, for a sufficiently large n such that $C_1 \max\left(s^2 ,\|\vecbeta^*\|_1^2,\|\vecbeta\|_1^4/s^2\right) \log (d/\delta)\leq n$ and $C_2\left(\sqrt{s\frac{\log (d/\delta)}{n}}+\sqrt{\frac{o}{n}}\right)\leq 1$, we can efficiently construct $\hat{\vecbeta}$ such that 
	\begin{align}
		\label{informal}
		\mbb{P}\left\{\|\hat{\vecbeta} -\vecbeta^*\|_2 \leq C_2\left(\sqrt{s\frac{\log (d/\delta)}{n}}+\sqrt{\frac{o}{n}}\right)\right\}\geq 1-3\delta,
	\end{align}
	where $C_1$ and $C_2$ are some constants depending on the properties of moments of $\vecx_i$ and $\xi_i$.
\end{theorem}
We see that, even when samples are contaminated by  malicious outliers, paying only $\sqrt{o/n}$ extra term is sufficient.
Similar results showing that, with appropriate estimators, the impact of outliers can be reduced, have been revealed in many previous works.
Some studies \cite{BakPra2021Robust, CheAraTriJorFlaBar2020Optimal} derived the information-theoretically optimal lower bound of estimating error of linear regression coefficients without sparsity when samples and noise are drawn from distributions with finite kurtosis and fourth moments, respectively and when the samples are contaminated by outliers.	The optimal lower bound in \cite{BakPra2021Robust, CheAraTriJorFlaBar2020Optimal} is $\sqrt{o/n}$ (for sufficiently large $n$) and our estimation error bound coincides with the optimal one about the term involving $\sqrt{o/n}$ up to constant factor.
In our situation, we require not only finite kurtosis but also finite $\mbb{E}(x_{ij_1}x_{ij_2}x_{ij_3}x_{ij_4})^2$ as an assumption for covariates. To remove the extra condition  from the assumption is a future task.

Standard lasso requires $n$ proportional to $s$ (\cite{RasWaiYu2010Restricted}), however
our estimator requires $n$ proportional to $s^2$.
A similar phenomenon  can be seen in \cite{WanBerSam2016Statistical,FanWanZhu2021Shrinkage,LiuSheLiCar2020High,BalDuLiSin2017Computationally, DiaKanKarPriSte2019Outlier} and so on.
Our method relies on the techniques of \cite{FanWanZhu2021Shrinkage, WanBerSam2016Statistical} and this is the cause for the stronger condition on the sample complexity of our estimator.
\cite{WanBerSam2016Statistical} considered sparse principal component analysis (PCA), revealing that there is no randomized polynomial time algorithm to estimate the top eigenvector in a scheme where $n$ is  proportional to $s$ (in \cite{WanBerSam2016Statistical}, $s$ is the number of non-zero elements of the top eigenvector of covariance matrices) under the assumptions of intractability of a variant of Planted Clique Problem. We leave the analysis in our situation for future work.

Finally, we note that the error bound or sample complexity of the estimators in \cite{CheCarMan2013Robust,BalDuLiSin2017Computationally,LiuSheLiCar2020High,FanWanZhu2021Shrinkage}, that dealt with sparse estimation of linear regression coefficients where the case that both covariates and noises are sampled from some heavy-tailed distributions or both the covariates and noises are contaminated by outliers, depend on norms of $\vecbeta^*$. 
Out estimator requires sufficiently large $n$ depending on $\|\vecbeta^*\|_1$ because our estimator use the technique developed in  \cite{FanWanZhu2021Shrinkage} to tame heavy-tailed covariates. 
To remove the effects of the true coefficient vector would be important.

In Section \ref{se:om}, we describe our estimation method and state our main result.
In Section \ref{sec:keyL}, we state key propositions without proofs, and the proof of the main theorem.
In Section \ref{sec:maindet}, we provide the proofs that are omitted in  Sections \ref{se:om} and  \ref{sec:keyL}.

\section{Method}
\label{se:om}
To estimate $\vecbeta^*$ in \eqref{model:adv}, we propose the following algorithm (ROBUST-SPARSE-ESTIMATION).
\cite{PenJogLoh2020robust} proposed some methods for estimating $\vecbeta^*$ from \eqref{model:adv}  when $\vecbeta^*$ has no sparsity, and derived sharp error bounds. The scheme of one of the methods in \cite{PenJogLoh2020robust} is 1. pre-processing covariates, and 2. executing the Huber regression with pre-processed covariates.
Our method is inspired by this one. However, we follow different pre-processings (PRUNING and COMPUTE-WEIGHT) and use the penalized Huber regression to enable us to tame the sparsity of $\vecbeta^*$. 
\begin{algorithm}
	\caption{ROBUST-SPARSE-ESTIMATION}
	\begin{algorithmic}[1]
			\label{ourmethod}
		\renewcommand{\algorithmicrequire}{\textbf{Input:}}
		\renewcommand{\algorithmicensure}{\textbf{Output:}}
		\REQUIRE $\left\{y_i,\vecX_i\right\}_{i=1}^n$ and the tuning parameters $\tau_\vecx,\,\lambda_*,\,\,\tau_{suc},\,\varepsilon,\,\lambda_o$ and $\lambda_s$
		\ENSURE  $\hat{\vecbeta}$
		\STATE $\{\tilde{\vecX}_i\}_{i=1}^n \leftarrow \text{PRUNING}(\left\{\vecX_i\right\}_{i=1}^n,\tau_\vecx)$
		\STATE $\left\{ \hat{w}_i\right\}_{i=1}^n \leftarrow \text{COMPUTE-WEIGHT}(\{\tilde{\vecX}_i\}_{i=1}^n ,\lambda_*,\tau_{suc},\varepsilon )$
		\STATE $\{ \hat{w}_i'\}_{i=1}^n \leftarrow \text{ROUNDING}(\left\{ \hat{w}_i\right\}_{i=1}^n )$
		\STATE $\hat{\vecbeta} \leftarrow \text{WEIGHTED-PENALIZED-HUBER-REGRESSION}\left(\{y_i,\tilde{\vecX}_i\}_{i=1}^n, \,\{\hat{w}'_i\}_{i=1}^n,\,\lambda_o,\,\lambda_s\right)$
	\end{algorithmic} 
\end{algorithm}
PRUNING is a procedure to make covariates bounded, which originated from \cite{FanWanZhu2021Shrinkage}, that deal with sparse estimations of vector/matrix when samples are drawn from a heavy-tailed distribution.
COMPUTE-WEIGHT relies on the semi-definite programming developed by \cite{BalDuLiSin2017Computationally}, which provides a method for  sparse PCA to be robust to outliers.
\cite{BalDuLiSin2017Computationally} considered a situation when samples that are drawn from Gaussian distribution and the samples are contaminated by outliers. PRUNING enables us to cast our heavy-tailed situation into the framework of \cite{BalDuLiSin2017Computationally}.

In the following Sections \ref{sec:p}, \ref{sec:cw}, \ref{sec:t} and \ref{sec:whr}, we describe the details of PRUNING, COMPUTE-WEIGHT, TRUNCATION and 
WEIGHTED-HUBER-REGRESSION, respectively.
Define
\begin{align}
	 r_o=\sqrt{\frac{o}{n}},\quad r_d = \sqrt{\frac{\log d}{n}},\quad r_\delta = \sqrt{\frac{\log (1/\delta)}{n}}.
\end{align}
\subsection{PRUNING}
\label{sec:p}
Define the $j$-th element of $\vecX_i$ as $\vecX_{i_j}$. For the choise of $\tau_\vecx$, see Remark \ref{rem:result}.
\begin{algorithm}[H]
	\caption{PRUNING}
	\label{alg:p}
	\begin{algorithmic}
	\REQUIRE{data $\{\vecX_i \}_{i=1}^n$, tuning parameter $\tau_\vecx$.}
	\ENSURE{pruned data $\{\tilde{\vecX}_i \}_{i=1}^n$.}\\
	{\bf For} {$i=1:n$}\\
		\ \ \ \ \ {\bf For} {$j=1:d$}\\
		\ \ \ \ \ \ \ \ \ $\tilde{\vecX}_{i_j}$ = $\mr{sgn}(\vecX_{i_j})\times \min\left(\vecX_{i_j},\tau_\vecx\right)$\\
	{\bf return}	$\{\tilde{\vecX}_i \}_{i=1}^n$.    
\end{algorithmic}
\end{algorithm}

\subsection{COMPUTE-WEIGHT}
\label{sec:cw}
For any matrix $M\in \mbb{R}^{d_1\times d_2} = \{m_{ij}\}_{1\leq i\leq d_1,1\leq j\leq d_2}$, define
\begin{align}
	\|M\|_1 = \sum_{i=1}^{d_1}\sum_{j=1}^{d_2}|m_{ij}|,\,\quad\|M\|_\infty =\max_{1\leq i\leq d_1,1\leq j\leq d_2}|m_{ij}|.
\end{align}
For a symmetric matrix $M$, we write $M\succeq 0$ if $M$ is positive semidefinite.
Define the following two convex sets:
\begin{align}
	\mathfrak{M}_r = \left\{M\in \mbb{R}^{d\times d} \,:\, \mr{Tr}(M) \leq  r^2,\,M\succeq 0\right\},\quad \mathfrak{U}_{\lambda} = \left\{U\in \mbb{R}^{d\times d} \,:\, \|U\|_{\infty} \leq \lambda,\, U\succeq 0\right\},
\end{align}
where $\mr{Tr}(M)$ for matrix $M$ is the trace of $M$.
To reduce the effects of outliers of covariates,
we require COMPUTE-WEIGHT to compute the weight vector $\hat{\vecw}= (\hat{w}_1,\cdots, \hat{w}_n)$ such that the following quantity is sufficiently small: 
\begin{align}
	\label{ine:adv-spect}
	\sup_{M\in \mathfrak{M}_{r}}\left(\sum_{i=1}^n\hat{w}_i \langle \tilde{\vecX}_i\tilde{\vecX}_i^\top M\rangle-\lambda_* \|M\|_1\right), 
\end{align}
where $\lambda_*$ is a tuning parameter.
Evaluation of  \eqref{ine:adv-spect} is required in the analysis of WEIGHTED-PENALIZED-HUBER-REGRESSION and the role of \eqref{ine:adv-spect} is revealed in the proof of Proposition \ref{p:main:out}.
For  COMPUTE-WEIGHT, we use a variant of Algorithm 4 of \cite{BalDuLiSin2017Computationally}.
For any vector $\vecv$, define the $\ell_\infty$ norm of $\vecv$ as $\|\vecv\|_\infty$ and 
define the probability simplex $\Delta^{n-1}$ as
\begin{align}
	\Delta^{n-1} = \left\{\vecw \in [0,1]^n: \sum_{i=1}^nw_i =1, \quad \|\vecw\|_\infty\leq \frac{1}{n(1-\varepsilon)}\right\}.
\end{align}
COMPUTE-WEIGHT is as follows.
\begin{algorithm}[H]
	\caption{COMPUTE-WEIGHT}
	\label{alg:cw0}
	\begin{algorithmic}
	\REQUIRE{data $\{\tilde{\vecX}_i \}_{i=1}^n$, tuning parameters $\lambda_*,\, \tau_{suc}$ and $\varepsilon$.}
	\ENSURE{weight estimate $\hat{\vecw} = \{\hat{w}_1,\cdots,\hat{w}_n\}$.}\\
	Let  $\hat{\vecw}$ be the solution to 
	\begin{align}
		\label{cw}
		\min_{\vecw \in \Delta^{n-1}} \max_{M\in \mathfrak{M}_r}\left(\sum_{i=1}^n w_i \langle\tilde{\vecX}_i \tilde{\vecX}_i^\top,M\rangle-\lambda_*\|M\|_1\right)
	\end{align}
	{\bf if} {the optimal value of \eqref{cw} $\leq \tau_{suc}$}\\
	\ \ \ \ \ {\bf return} {$\hat{\vecw}$}\\
	{\bf else}    \\
	\ \ \ \ \ {\bf return} {$fail$}\\
\end{algorithmic}
\end{algorithm}
We note that, from the arguments of \cite{WanBerSam2016Statistical, Nem2004Prox,Nes2005Smooth}, we have
\begin{align}
	\label{ine:opt-transform}
		\min_{\vecw \in \Delta^{n-1}}\max_{M\in \mathfrak{M}_{r}} \left( \sum_{i=1}^n w_i\langle \tilde{\vecX}_i \tilde{\vecX}_i^\top,M\rangle-\lambda_*\|M\|_1\right)= \min_{\vecw \in \Delta^{n-1}} \min_{U \in \mathfrak{U}_{\lambda_*}}\max_{M\in \mathfrak{M}_{r}}\left\langle \sum_{i=1}^n w_i \tilde{\vecX}_i \tilde{\vecX}_i^\top- U,M\right\rangle.
\end{align}
COMPUTE-WEIGHT and Algorithm 4 of \cite{BalDuLiSin2017Computationally} are very similar
and the difference are the convex constraints and the values of parameters.
For any fixed $\vecw$, our objective function and the constraints are the same as the ones in Section 3 of \cite{WanBerSam2016Statistical} except for the values of the tuning parameters, and we can efficiently find the optimal $M \in \mathfrak{M}_{r}$.
Therefore, COMPUTE-WEIGHTS can be solved efficiently for the same reason as Algorithm 4 of \cite{BalDuLiSin2017Computationally}.

To analyze COMPUTE-WRIGHT, we introduce the following proposition.
The poof of the following proposition is provided in Section \ref{sec:maindet}. 
Define $\sigma_{\vecx,4}^4 = \max_{1\leq j \leq  d}\mbb{E}\vecx_{i_j}^4$ and $\sigma_{\vecx,2}^2 = \max_{1\leq j \leq  d} \mbb{E}\vecx_{i_j}^2$. Define $\Sigma = \mbb{E}\vecx_i\vecx_i^\top$, and 
$\tilde{\vecx}_{i_j}$ = $\mr{sgn}(\vecx_{i_j})\times \min\left(\vecx_{i_j},\tau_\vecx\right)$.

\begin{proposition}
	\label{p:cwpre}
	Assume that $\{\vecx_i\}_{i=1}^n$  is a sequence of i.i.d. $d\, (\geq 3)$-dimensional random vectors with zero mean and with finite $\sigma_{\vecx,4}^4$.
	For any matrix $M \in \mathfrak{M}_{r}$, with probability at least $1-\delta$, we have
	\begin{align}
		\label{ine:cwpre}
	\sum_{i=1}^n \frac{\left\langle\tilde{\vecx}_i \tilde{\vecx}_i^\top,M\right\rangle}{n}\leq \left\{\sqrt{2}\sigma_{\vecx,4}^2(r_d+r_\delta)+\tau_{\vecx}^2 (r_d^2+r_\delta^2)+2\frac{\sigma_{\vecx,4}^4}{\tau_{\vecx}^2}\right\}\|M\|_1+\|\Sigma\|_{\mr{op}}r^2,
	\end{align}
	where $\|\Sigma\|_{\mr{op}}$ is the operator norm of $\Sigma$.
\end{proposition}
Define $\tau_{suc}'$, $ \lambda_*'$ and $\{w_i^\circ \}_{i=1}^n$ as
\begin{align}
	\tau_{suc}' =\frac{\|\Sigma\|_{\mr{op}}}{1-\varepsilon}r^2,\, \lambda_*'=\frac{1}{1-\varepsilon}\left\{\sqrt{2}\sigma_{\vecx,4}^2(r_d+r_\delta)+\tau_{\vecx}^2 (r_d^2+r_\delta^2)+2\frac{\sigma_{\vecx,4}^4}{\tau_{\vecx}^2}\right\},\,	w_i^\circ = \begin{cases}
		\frac{1}{n(1-\varepsilon)}&i\in\mc{I} \\
		0 & i\in\mc{O}
		\end{cases}.
\end{align}
From Proposition \ref{p:cwpre}, when $\tau_{suc}' \leq \tau_{suc},\,\lambda_*'\leq \lambda_*$ and $o/n \leq \varepsilon$ hold, we have, with probability at least $1-\delta$, 
\begin{align}
	\label{ine:optM}
	\max_{M\in \mathfrak{M}_{r}}\left(\sum_{i =1}^n \hat{w}_i \langle \tilde{\vecX}_i\tilde{\vecX}_i^\top, M\rangle -\lambda_*\|M\|_1\right)&\stackrel{(a)}{\leq} \max_{M\in \mathfrak{M}_{r}}\left(\sum_{i \in \mc{I}}w^\circ_i \left\langle \tilde{\vecX}_i\tilde{\vecX}_i^\top ,M\right\rangle -\lambda_*\|M\|_1\right)\nonumber\\
	&=\max_{M\in \mathfrak{M}_{r}}\left(\sum_{i \in \mc{I}}w_i^\circ  \left\langle \tilde{\vecx}_i\tilde{\vecx}_i^\top, M\right\rangle -\lambda_*\|M\|_1\right)\nonumber\\
	&\stackrel{(b)}\leq\max_{M\in \mathfrak{M}_{r}}\left(\sum_{i =1}^n \frac{1}{n(1-\varepsilon)}  \left\langle \tilde{\vecx}_i\tilde{\vecx}_i^\top, M\right\rangle -\lambda_*\|M\|_1\right)\nonumber\\
	&\stackrel{(c)}{\leq} \max_{M\in \mathfrak{M}_{r}}\left(\sum_{i =1}^n\frac{1}{n}  \left\langle \tilde{\vecx}_i\tilde{\vecx}_i^\top, M\right\rangle -\lambda_*'(1-\varepsilon)\|M\|_1  \right) \times \frac{1}{1-\varepsilon}\nonumber\\
	&\leq \tau_{suc}',
\end{align}
where (a) follows from the optimality of $\hat{w}_i$, $o/n\leq \varepsilon$ and $\{w_i^\circ\}_{i=1}^n \in \Delta^{n-1}$, (b) follows from positive semi-definiteness of $M$, and (c) follows from $\lambda_*'\leq \lambda_*$. Therefore, from $\tau_{suc}'\leq \tau_{suc}$, we see that COMPUTE-WEIGHT succeed and return $\hat{\vecw}$ with probability at least $1-\delta$.
We note that \eqref{ine:optM} is used in the proof of Proposition \ref{p:main:out}.

\subsection{TRUNCATION}
\label{sec:t}
TRUNCATION is a discretization of $\left\{ \hat{w}_i\right\}_{i=1}^n$. TRUNCATION makes  it easy to analyze the estimator. We see that the number of $\hat{w}_1,\cdots,\hat{w}_n$ rounded at zero is at most $2\varepsilon n $ from Proposition \ref{l:w2}. 
\begin{algorithm}[H]
	\caption{TRUNCATION}
	\label{alg:t}
	\begin{algorithmic}
	\REQUIRE{weight vector $\hat{\vecw} = \{\hat{w}_i\}_{i=1}^n$.}
	\ENSURE{rounded weight vector $\hat{\vecw}'= \{\hat{w}'_i\}_{i=1}^n$.}\\
	{\bf For} {$i=1:n$}\\
		\ \ \ \ \ {\bf if} $\hat{w}_i \geq \frac{1}{2n}$\\
		\ \ \ \ \ \ \ \ \ \ $\hat{w}'_i = \frac{1}{n}$\\
		\ \ \ \ \ {\bf else}\\
		\ \ \ \ \ \ \ \ \ \ $\hat{w}'_i = 0$\\
	 {\bf return}	$\hat{\vecw}'$.    
\end{algorithmic}
\end{algorithm}

\subsection{WEIGHTED-PENALIZED-HUBER-REGRESSION}
\label{sec:whr}
WEIGHTED-PENALIZED-HUBER-REGRESSION is a type of regression using the Huber loss with $\ell_1$ penalization.
Define Huber loss function $H(t)$,
\begin{align}
H(t) = \begin{cases}
|t| -1/2 & (|t| > 1) \\
t^2/2  & (|t| \leq 1)
\end{cases}.
\end{align}
and let
\begin{align}
	h(t) =	\frac{d}{dt} H(t) =   \begin{cases}
	t\quad &(|t|> 1)\\
	\mr{sgn}(t)\quad &(|t| \leq 1)
\end{cases}.
\end{align}
We consider the following optimization problem.
For any vector $\vecv$, define the $\ell_1$ norm of $\vecv$ as $\|\vecv\|_1$.
\begin{algorithm}[H]
	\caption{WEIGHTED-PENALIZED-HUBER-REGRESSION}
	\label{alg:WH}
	\begin{algorithmic}
	\REQUIRE{data $\left\{y_i,\tilde{\vecX}_i\right\}_{i=1}^n$, rounded weight vector $\hat{\vecw}' = \{\hat{w}'_i\}_{i=1}^n$ and tuning parameters $\lambda_o, \lambda_s$.}
	\ENSURE{estimator $\hat{\vecbeta}$.}\\
	Let $\hat{\vecbeta}$ be the solution to 
	\begin{align}
		\argmin_{\vecbeta  \in \mbb{R}^d} \sum_{i=1}^n \lambda_o^2 H\left(n\hat{w}_i'\frac{y_i-\tilde{\vecX}_i^\top\vecbeta}{\lambda_o\sqrt{n}}\right)+\lambda_s\|\vecbeta\|_1,
		\end{align}
	 {\bf return}	$\hat{\vecbeta}$.    
\end{algorithmic}
\end{algorithm}
We note that many studies e.g., \cite{NguTra2012Robust,SheOwe2011Outlier,DalTho2019Outlier,CheZho2020Robust,Chi2020Erm, PenJogLoh2020robust} imply that the Huber loss is effective for linear regression under heavy-tailed noises or the existence of  outliers.

\subsection{Results}
\label{sec:results}
We introduce our assumption after first introducing the notion of finite kurtosis distribution.
\begin{definition}[Finite kurtosis distribution]
	\label{def:fk}
	A random vector $\vecz\in \mbb{R}^d$ with zero mean is said to have finite kurtosis distribution if for every $\vecv \in \mbb{R}^d$, 
	\begin{align}
		\label{ine:fc}
		\mbb{E}(\vecv^\top\vecz)^4 \leq K^4 \{\mbb{E}(\vecv^\top\vecz)^2\}^2 .
	\end{align}
\end{definition}
We note that the finite kurtosis distribution sometimes referred as $L_4$-$L_2$ norm equivalence \cite{MenZhi2020Robust} or $L_4$-$L_2$ hyper-contractivity \cite{CheAraTriJorFlaBar2020Optimal}. 
Define  the minimum singular value of $\Sigma^\frac{1}{2}$ as $\lambda_{\Sigma}$.
\begin{assumption}
	\label{a:1}
	Assume that 
	\begin{itemize}
		\item[(i)] $\{\vecx_i\}_{i=1}^n$  is a sequence of i.i.d. $d\, (\geq 3)$-dimensional random vectors with zero mean, with finite kurtosis, $\mbb{E}(x_{ij_1}x_{ij_2}x_{ij_3}x_{ij_4})^2 \leq \sigma_{\vecx,8}^8$ for any $1\leq j_1,j_2,j_3,j_4\leq d$, and $\lambda_\Sigma>0$, and for simplicity, assume $1\geq \lambda_\Sigma$,
		\item[(ii)] $\{\xi_i\}_{i=1}^n$ is a sequence of i.i.d. random variables whose absolute moments are bounded by $\sigma$,
		\item[(iii)] $\mbb{E}h\left(\frac{\xi_i}{\lambda_o\sqrt{n}}\right) \times \vecx_i=0$.
	\end{itemize}
\end{assumption}
\begin{remark}
	The condition (iii) in Assumption \ref{a:1} is a weaker condition than the independence between $\{\xi_i\}_{i=1}^n$ and $\{\vecx_i\}_{i=1}^n$.
\end{remark}
Under Assumption \ref{a:1}, we have the following theorem.
Define 
\begin{align}
	r_{\vecx, d} =(\sigma_{\vecx,2}+1)r_d+\tau_\vecx r_d^2,\quad r_{\vecx, \delta}=(\sigma_{\vecx,2}+1)r_\delta+ \tau_{\vecx} r_\delta^2,\quad r_{d,\delta} = r_{\vecx, d}+r_{\vecx, \delta}.
\end{align}
\begin{theorem}
	\label{t:main}
	Suppose that Assumption \ref{a:1} holds. 
	Suppose that parameters $\tau_\vecx,\,\lambda_*,\,\,\tau_{suc},\,\varepsilon,\,\lambda_o,\,\lambda_s$ and $r$ satisfy 
	\begin{align}
		\label{ine:tp1}
		\tau^2_\vecx&\geq\max\left\{ \frac{\|\vecbeta^*\|_1^2\sigma_{\vecx,8}^8\|\Sigma^\frac{1}{2}\|_{\mr{op}}^2}{s\lambda_o^2n},\left(\frac{\|\vecbeta^*\|_1}{\lambda_o\sqrt{n}}\right)^\frac{1}{2},\frac{108\sigma_{\vecx,4}^4s}{\lambda_{\Sigma}^2},\left(\|\vecbeta^*\|_1\sigma_{\vecx,8}^4\right)^\frac{2}{3},\frac{9\sigma_{\vecx,8}^4s}{K^2}\right\},\\
		\label{ine:tp1-2}
		\lambda_*&\geq \lambda_*',\quad\tau_{suc}=c_{suc}\tau_{suc}',\quad \frac{1}{2}> \varepsilon \geq \max\left\{\frac{o}{n},\frac{1}{n}\right\}\\
		\label{ine:tp2}
		\lambda_o  \sqrt{n} &\geq \max\left\{\frac{16K\|\Sigma^{\frac{1}{2}}\|_{\mr{op}}}{\lambda_{\Sigma}^2},\frac{300K^4\|\Sigma^{\frac{1}{2}}\|_{\mr{op}}^4(\sigma+1)}{\lambda_{\Sigma}^4},4K^2\|\Sigma^{\frac{1}{2}}\|_{\mr{op}}^2\right\},\\
		\label{ine:tp3}
		&\lambda_s\geq  c_s\lambda_o \sqrt{n} \left\{r_{d,\delta}+\frac{\sigma_{\vecx,2}\sigma_{\vecx,4}^2+\sigma_{\vecx,8}^8}{\tau_\vecx^2}+(\sqrt{\lambda_*} + \sqrt{\lambda_*'})r_o+\sqrt{\lambda_*'\varepsilon}+\frac{1}{ \sqrt{s}}\|\Sigma^\frac{1}{2}\|_{\mr{op}}\left( \sqrt{c_{suc}} r_o +\sqrt{\varepsilon}\right)+\frac{1}{\tau_\vecx^2}\right\},\\
		\label{ine:tp4}
		&r \geq C_s \frac{\sqrt{s}\lambda_s}{\lambda^2_\Sigma} +  C_s\times \nonumber \\
		& \frac{\lambda_o\sqrt{n}}{\lambda_\Sigma^2} \left[\sqrt{s}\left\{r_{d,\delta}+\frac{\sigma_{\vecx,2}\sigma_{\vecx,4}^2+\sigma_{\vecx,8}^8}{\tau_\vecx^2}+(\sqrt{\lambda_*} + \sqrt{\lambda_*'})r_o+\sqrt{\lambda_*'\varepsilon}\right\}+ \|\Sigma^\frac{1}{2}\|_{\mr{op}}\left(\sqrt{c_{suc}} r_o +\sqrt{\varepsilon}\right)+\frac{\sqrt{s}}{\tau_\vecx^2}\right],
	\end{align} 
	 and $r\leq 1$, where $c_s, C_s$ and $c_{suc}$ are sufficiently large  numerical constants such that $c_s \geq 6,\, C_s \geq 300 $ and $c_{suc} \geq 1$
	Then,	with probability at least $1-3\delta$,
	the output of ROBUST-SPARSE-ESTIMATION $\hat{\vecbeta}$ satisfies
	\begin{align}
		\label{ine:result}
		\|\hat{\vecbeta} -\vecbeta^*\|_2 & \leq r.
	\end{align}
\end{theorem}
\begin{remark}
	\label{rem:result}
	We consider the conditions \eqref{ine:tp1}-\eqref{ine:tp4} and the result \eqref{ine:result} in detail.  For simplicity, assume that $\max\{o/n,1/n\} = o/n$.
	Let $\tau_{\vecx} = 1/(r_d^2+r_\delta^2)^\frac{1}{4}$ and assume that, for the tuning parameters, the lower bounds of the inequalities in \eqref{ine:tp1}-\eqref{ine:tp4} hold and $\tau_{suc} = \tau_{suc}'$ with $c_{suc} =1 $ holds. Then, 
	\begin{align}
		\tau^2_\vecx&\geq\max\left\{ \frac{\|\vecbeta^*\|_1^2\sigma_{\vecx,8}^8\|\Sigma^\frac{1}{2}\|_{\mr{op}}^2}{s\lambda_o^2n},\left(\frac{\|\vecbeta^*\|_1}{\lambda_o\sqrt{n}}\right)^\frac{1}{2},\frac{108\sigma_{\vecx,4}^4s}{\lambda_{\Sigma}^2},\left(\|\vecbeta^*\|_1\sigma_{\vecx,8}^4\right)^\frac{2}{3},\frac{9\sigma_{\vecx,8}^4s}{K^2}\right\}
	\end{align}
	means that, $n$ is sufficiently large so that
	\begin{align}
	\max\left\{ \frac{\|\vecbeta^*\|_1^2\sigma_{\vecx,8}^8\|\Sigma^\frac{1}{2}\|_{\mr{op}}^2}{s\lambda_o^2n},\frac{\|\vecbeta^*\|_1}{\lambda_o\sqrt{n}},\frac{108\sigma_{\vecx,4}^4s}{\lambda_{\Sigma}^2},\|\vecbeta^*\|_1\sigma_{\vecx,8}^4,\frac{9\sigma_{\vecx,8}^4s}{K^2},1\right\}\sqrt{\log (d/\delta)}\leq \sqrt{n}
	\end{align}
	holds, where we use
	\begin{align}
		\left(\frac{\|\vecbeta^*\|_1}{\lambda_o\sqrt{n}}\right)^\frac{1}{2}\leq \max\left\{\frac{\|\vecbeta^*\|_1}{\lambda_o\sqrt{n}},1\right\},\quad \left(\|\vecbeta^*\|_1\sigma_{\vecx,8}^4\right)^\frac{2}{3}\leq \max\left\{\|\vecbeta^*\|_1\sigma_{\vecx,8}^4,1\right\}
	\end{align}
	from Young's inequality.
	In addition, assume that 
	\begin{align}
		(\sqrt{2}\sigma_{\vecx,4}^2+1+2\sigma_{\vecx,4}^4)(r_d+r_\delta)\sqrt{s}<(1-\varepsilon)\|\Sigma^\frac{1}{2}\|_{\mr{op}}
	\end{align}
	holds and this means  $(\sqrt{s\lambda_*'}=)\sqrt{s\lambda_*}<\|\Sigma^\frac{1}{2}\|_{\mr{op}}$ holds.
	Then, from $1/(1-\varepsilon)\leq 2$, we see that \eqref{ine:result} becomes
	\begin{align}
		\label{rec}
		&\|\hat{\vecbeta} -\vecbeta^*\|_2   \leq C\frac{\lambda_o\sqrt{n}}{\lambda_\Sigma^2}\left[ \left\{\left(\sigma_{\vecx,2}\sigma_{\vecx,4}^2+\sigma_{\vecx,8}^8+1\right)(r_d+r_\delta)+\sigma_{\vecx,4}^4(r_d^2+r_\delta^2) \right\}\sqrt{s}+\|\Sigma^\frac{1}{2}\|_{\mr{op}}r_o\right],
	\end{align}
	where $C$ is a numerical constant.
	We see that this gives the same result as in Theorem \ref{theoreminformal1}.
\end{remark}
\begin{remark}
	We do not optimize the numerical constants $C_s$ and $c_s$ in Theorem \ref{t:main}.
\end{remark}

\section{Key propositions}
\label{sec:keyL}
The poof of Propositions in this section are provided in Section \ref{sec:maindet}. 
First, we introduce our main proposition (Proposition \ref{p:main}), that is stated in a deterministic form.
We see that the conditions in Proposition \ref{p:main} are satisfied with high probability by Proposition \ref{p:main1}-\ref{l:w2} under Assumption \ref{a:1}.
Let 
\begin{align} 
		r_{\vecv,i} =\hat{w}_i'n\frac{y_i-\tilde{\vecX}_i^\top \vecv}{\lambda_o \sqrt{n}},\quad \xi_{\lambda_o,i} = \frac{(\vecx_i-\tilde{\vecx}_i)^\top\vecbeta^*+\xi_i}{\lambda_o \sqrt{n}},\quad \tilde{X}_{\vecv,i} = \frac{\tilde{\vecX}_i^\top \vecv}{\lambda_o\sqrt{n}},\quad \tilde{x}_{\vecv,i} = \frac{\tilde{\vecx}_i^\top \vecv}{\lambda_o\sqrt{n}},
\end{align}
and for $\eta \in (0,1)$,
\begin{align}
	\vectheta = \hat{\vecbeta}-\vecbeta^*,\quad \vectheta_\eta = (\hat{\vecbeta}-\vecbeta^*)\eta
\end{align}
\begin{proposition}
		\label{p:main}
		Suppose that $\{\xi_i,\tilde{\vecX}_i,\hat{w}_i'\}_{i=1}^n$ and $\lambda_o$ satisfies \eqref{ine:det:main1} and \eqref{ine:det:main3}: for any $\eta \in (0,1)$ such that $\|\vectheta_\eta\|_2=r \leq 1$,
		\begin{align}
			\label{ine:det:main1}
			&\left| \lambda_o\sqrt{n}\sum_{i=1}^n \hat{w}_i'h(r_{\vecbeta^*,i}) \tilde{\vecX}_i^\top \vectheta_\eta \right| \leq  r_{a,2} \|\vectheta_\eta\|_2 + r_{a,1}\|\vectheta_\eta\|_1,\\
			\label{ine:det:main3}
			&		b_1 \|\vectheta_\eta\|_2^2-r_{b,2}\|\vectheta_\eta\|_2 - r_{b,1}\leq \sum_{i=1}^n \lambda_o\sqrt{n} \hat{w}_i'\left\{-h(r_{\vecbeta^*+\vectheta_\eta,i}) +h(r_{\vecbeta^*,i})\right\} \tilde{\vecX}_i^\top \vectheta_\eta,
		\end{align}
		where $b_1>0, r, r_{a,2},r_{a,1}, r_{b,2},r_{b,1} \geq 0$ are some numbers.
		Suppose that $\lambda_s$ satisfies
		\begin{align}
			\label{ine:det:main4}
			\lambda_s-C_s>0,\quad \frac{\lambda_s+C_s}{\lambda_s-C_s}\leq 2,\text{ where }C_s =\frac{r_{a,2}}{3 \sqrt{s}} + r_{a,1}.
		\end{align} 
		Then, for $r$ such that 
		\begin{align}
				\label{ine:det:main5}
			\frac{r_{a,2}+r_{b,2} + 3(r_{a,1}+\lambda_s) \sqrt{s}+\sqrt{b_1r_{b,1}}}{b_1}  < r,
		\end{align}
		the output of ROBUST-SPARSE-ESTIMATION $\hat{\vecbeta}$ satisfies $\|\hat{\vecbeta} -\vecbeta^*\|_2  \leq r$.
\end{proposition}
In the remaining part of Section \ref{sec:keyL}, we introduce some propositions to prove  \eqref{ine:det:main1} and \eqref{ine:det:main3}
are satisfied with high probability for appropriate values of $b_1>0, r_{a,1},r_{a,2}, r_{b,1},r_{b,2}$ under the assumptions in Theorem \ref{t:main}.
\begin{proposition}
	\label{p:main1}
	Suppose that Assumption \ref{a:1} holds, and suppose that 
	\begin{align}
		\tau_\vecx^2 \geq \max\left\{ \frac{\|\vecbeta^*\|_1^2\sigma_{\vecx,8}^8\|\Sigma^\frac{1}{2}\|_{\mr{op}}^2}{s\lambda_o^2n},\left(\frac{\|\vecbeta^*\|_1}{\lambda_o\sqrt{n}}\right)^\frac{1}{2}\right\}
	\end{align}
	holds.
	Then, for any $\vecv\in \mbb{R}^d$, we have 
	\begin{align}
		\label{ine:main1}
		\left| \sum_{i=1}^n \frac{1}{n} h(\xi_{\lambda_o,i}) \langle \tilde{\vecx}_i,\vecv\rangle  \right|\leq \|\vecv\|_1\left\{\sqrt{2}r_{d,\delta}+\frac{\sigma_{\vecx,2}\sigma_{\vecx,4}^2+\sigma_{\vecx,8}^8}{\tau_\vecx^2}\right\} + \frac{\sqrt{s}}{\tau_{\vecx}^2}\|\vecv\|_2
	\end{align}
	with probability at least $1-\delta$.
\end{proposition}

\begin{proposition}
	\label{p:main:out}
	Suppose that \eqref{ine:cwpre} holds and COMPUTE-WEIGHT returns $\hat{w}$.
	For any  $\|\vecu\|\in \mbb{R}^n$ such that $\|\vecu\|_\infty \leq c$ for a numerical constant $c$ and for any $\vecv \in \mbb{R}^d$ such that $\|\vecv\|_2 = r$, we have
	\begin{align}
		\left|\sum_{i \in \mc{O}}\hat{w}'_iu_i \tilde{\vecX}_i^\top\vecv \right| \leq \sqrt{2}c r_o \sqrt{\tau_{suc}} + \sqrt{2} cr_o\sqrt{\lambda_*} \|\vecv\|_1.
	\end{align}
\end{proposition}

Define $I_m$ as the index set such that $|I_m|$ = $m$.
\begin{proposition}
	\label{p:main:out2}
	Suppose that \eqref{ine:cwpre} holds.
	For any  $\vecu \in \mbb{R}^n$ such that $\|\vecu\|_\infty \leq c$ for a numerical constant $c$ and for any $\vecv \in \mbb{R}^{d}$ such that $\|\vecv\|_2 = r$, we have
	\begin{align}
		\left|\sum_{i \in I_m}\frac{1}{n}u_i \tilde{\vecx}_i^\top\vecv \right|\leq  c\sqrt{\frac{m}{n}}\|\Sigma^\frac{1}{2}\|_{\mr{op}}r + c\sqrt{\frac{m}{n}}\sqrt{\lambda_*'}\|\vecv\|_1.
	\end{align}
\end{proposition}

\begin{proposition}
	\label{p:main:sc}
	Suppose that Assumption \ref{a:1} holds. 
	Let  
	\begin{align}
	\mc{R}_{\vecv} = \left\{ \vecv \in \mbb{R}^d\, |\,  \|\vecv\|_2 = r\leq 1\, , \|\vecv\|_1\leq 3 \sqrt{s}\|\vecv\|_2\right\}.
	\end{align}
	Suppose that 
	\begin{align}
		\label{ine:r}
		\lambda_o  \sqrt{n} &\geq \max\left\{\frac{16K\|\Sigma^{\frac{1}{2}}\|_{\mr{op}}}{\lambda_{\Sigma}^2},\frac{300K^4\|\Sigma^{\frac{1}{2}}\|_{\mr{op}}^4(\sigma+1)}{\lambda_{\Sigma}^4},4K^2\|\Sigma^{\frac{1}{2}}\|_{\mr{op}}^2\right\},\nonumber \\
		\tau^2_\vecx&\geq\max\left\{\frac{108\sigma_{\vecx,4}^4s}{\lambda_{\Sigma}^2},\left(\|\vecbeta^*\|_1\sigma_{\vecx,8}^4\right)^\frac{2}{3},\frac{9\sigma_{\vecx,8}^4s}{K^2}\right\}.
	\end{align} 
	Then, for any $\vecv \in \mc{R}_{\vecv}$, with probability at least $1-\delta$,  we have
	\begin{align}
		\label{ine:sc}
		&\sum_{i=1}^n \frac{\lambda_o}{\sqrt{n}} \left\{-h(\xi_{\lambda_o,i}-\tilde{x}_{\vecv,i})+h(\xi_{\lambda_o,i})\right\}\tilde{\vecx}_i^\top \vecv \geq \frac{\lambda_{\Sigma}^2}{6}\|\vecv\|_2^2-24\lambda_o\sqrt{n} \sqrt{s}r_{d,\delta}\|\vecv\|_2- 18\lambda_o^2n r_\delta^2.
	\end{align}
\end{proposition}

Let $I_{<}$ and $I_{\geq}$ be  the sets of the indices such that $w_i < 1/(2n)$  and $w_i \geq 1/(2n)$, respectively. 
The following proposition is also used to prove the main theorem.
\begin{proposition}
	\label{l:w2}
	Suppose $0<\varepsilon <1 $. Then, for any $\vecw \in \Delta^{n-1}$, we have $|I_{<}| \leq 2n\varepsilon$.
\end{proposition}

\subsection{Proof of the main theorem}
\label{sec:pmt}
We confirm \eqref{ine:det:main1}-\eqref{ine:det:main5} is satisfied with probability at least $1-3\delta$ under the assumptions in Theorem \ref{t:main}.
First, from assumptions in Theorem \ref{t:main}, we see that, from union bound, Propositions \ref{p:cwpre}, \ref{p:main1} - \ref{p:main:sc} hold with probability at least $1-3\delta$.
In the remaining part of Section \ref{sec:pmt},  Propositions \ref{p:cwpre}, \ref{p:main1} - \ref{p:main:sc} are assumed to hold.

\subsubsection{Confirmation of \eqref{ine:det:main1}}
\label{subsubsec:main1}
	We confirm \eqref{ine:det:main1}.
	From triangular inequality, we have
		\begin{align}
			&\left|\sum_{i=1}^n \hat{w}_i'h(r_{\vecbeta^*,i}) \tilde{\vecX}_i^\top \vectheta_\eta \right| \nonumber \\
			&=\left|\sum_{i \in \mc{I}} \hat{w}_i'h(r_{\vecbeta^*,i}) \tilde{\vecX}_i^\top \vectheta_\eta + \sum_{i \in \mc{O}} \hat{w}_i'h(r_{\vecbeta^*,i}) \tilde{\vecX}_i^\top \vectheta_\eta \right|\nonumber \ \\
			&=\left|\sum_{i \in \mc{I}} \hat{w}_i'h(\xi_{\lambda_o,i}) \tilde{\vecX}_i^\top \vectheta_\eta + \sum_{i \in \mc{O}} \hat{w}_i'h(r_{\vecbeta^*,i}) \tilde{\vecX}_i^\top \vectheta_\eta \right|\nonumber \ \\
			&= \left| \sum_{i=1}^n\frac{1}{n}h(\xi_{\lambda_o,i}) \tilde{\vecx}_i^\top \vectheta_\eta  +\sum_{i\in \mc{O}} \hat{w}_i'h(r_{\vecbeta^*,i}) \tilde{\vecX}_i^\top \vectheta_\eta -\sum_{i\in  \mc{O} \cup \left(\mc{I} \cap I_{<}\right)}\frac{1}{n}h(\xi_{\lambda_o,i})\tilde{\vecx}_i^\top \vectheta_\eta \right|\nonumber \\\
			&\leq \left| \sum_{i=1}^n\frac{1}{n}h(\xi_{\lambda_o,i}) \tilde{\vecx}_i^\top \vectheta_\eta  \right| +\left|\sum_{i\in \mc{O}} \hat{w}_i'h(r_{\vecbeta^*,i}) \tilde{\vecX}_i^\top \vectheta_\eta \right| +\left| \sum_{i\in  \mc{O} \cup \left(\mc{I} \cap I_{<}\right)}\frac{1}{n}h(\xi_{\lambda_o,i})\tilde{\vecx}_i^\top \vectheta_\eta \right|.
		\end{align}
		We note that $|h(\cdot)|\leq 1$ and from Proposition  \ref{l:w2}, $ |\mc{O} \cup \left(\mc{I} \cap I_{<}\right)|\leq o+2\varepsilon n$.
		Therefore, from  Propositions \ref{p:main1} - \ref{p:main:out2} with $c=1$, we have
		\begin{align}
			&\left|\sum_{i=1}^n \hat{w}_i'h(r_{\vecbeta^*,i}) \tilde{\vecX}_i^\top \vectheta_\eta \right| \nonumber  \\
			&\leq \left\{2r_{d,\delta}+2\frac{\sigma_{\vecx,2}\sigma_{\vecx,4}^2+\sigma_{\vecx,8}^8}{\tau_\vecx^2}+ \sqrt{2} r_o\sqrt{\lambda_*}+ \sqrt{\frac{o+2\varepsilon n}{n}}\sqrt{\lambda_*'}\right\} \|\vectheta_\eta\|_1 + \left\{\sqrt{2} r_o \sqrt{\tau_{suc}}  +\sqrt{\frac{o+2\varepsilon n}{n}}\|\Sigma^\frac{1}{2}\|_{\mr{op}}r+\frac{\sqrt{s}}{\tau_\vecx^2}r\right\}\nonumber\\
			&\stackrel{(a)}{\leq} \left\{2r_{d,\delta}+2\frac{\sigma_{\vecx,2}\sigma_{\vecx,4}^2+\sigma_{\vecx,8}^8}{\tau_\vecx^2}+(\sqrt{2\lambda_*} + \sqrt{\lambda_*'})r_o+\sqrt{2\lambda_*'\varepsilon}\right\} \|\vectheta_\eta\|_1 + \left\{\left(3\sqrt{c_{suc}}r_o +\sqrt{2\varepsilon}\right)\|\Sigma^\frac{1}{2}\|_{\mr{op}}+\frac{\sqrt{s}}{\tau_\vecx^2}\right\}\|\vectheta_\eta\|_2,
		\end{align}
where (a) follows from $\tau_{suc} = c_{suc}\|\Sigma\|_{\mr{op}}r^2/(1-\varepsilon)$, $1\leq \sqrt{c_{suc}/(1-\varepsilon)}$, $1/(1-\varepsilon)\leq 2$, and $r = \|\vectheta_\eta\|_2$.
		We see that \eqref{ine:det:main1} holds with 
		\begin{align}
			\label{a}
			r_{a,1}&=\lambda_o \sqrt{n}\left\{2r_{d,\delta}+\frac{\sigma_{\vecx,2}\sigma_{\vecx,4}^2+\sigma_{\vecx,8}^8}{\tau_\vecx^2}+(\sqrt{2\lambda_*} + \sqrt{\lambda_*'})r_o+\sqrt{2\varepsilon\lambda_*'}\right\},\nonumber \\
			r_{a,2}&=\lambda_o \sqrt{n}\left(3\sqrt{c_{suc}} r_o +\sqrt{2\varepsilon}\right)\|\Sigma^\frac{1}{2}\|_{\mr{op}}+\frac{\sqrt{s}}{\tau_\vecx^2}.
		\end{align}
	\subsubsection{Confirmation of \eqref{ine:det:main4}}
	From \eqref{a}, we see
	\begin{align}
		C_s&=\frac{r_{a,2}}{3 \sqrt{s}} + r_{a,1}\nonumber \\
		&\leq \frac{\lambda_o \sqrt{n}}{3 \sqrt{s}} \left\{\left( 3\sqrt{c_{suc}} r_o +\sqrt{2\varepsilon}\right)\|\Sigma^\frac{1}{2}\|_{\mr{op}}+\frac{\sqrt{s}}{\tau_\vecx^2}\right\}+\lambda_o \sqrt{n}\left\{ 2r_{d,\delta}+\frac{\sigma_{\vecx,2}\sigma_{\vecx,4}^2+\sigma_{\vecx,8}^8}{\tau_\vecx^2}+(\sqrt{2\lambda_*} + \sqrt{\lambda_*'})r_o+\sqrt{2\lambda_*'\varepsilon}\right\}\nonumber \\
		&\leq 2\lambda_o \sqrt{n} \left\{\frac{1}{ \sqrt{s}}\left\{\left(\sqrt{c_{suc}} r_o +\sqrt{\varepsilon}\right)\|\Sigma^\frac{1}{2}\|_{\mr{op}}+\frac{\sqrt{s}}{\tau_\vecx^2}\right\}+ r_{d,\delta}+\frac{\sigma_{\vecx,2}\sigma_{\vecx,4}^2+\sigma_{\vecx,8}^8}{\tau_\vecx^2}+(\sqrt{\lambda_*} + \sqrt{\lambda_*'})r_o+\sqrt{\lambda_*'\varepsilon}\right\}.
	\end{align} 
	Therefore, we see, for a sufficiently large constant $c_s$ such that $c_s\geq 6$, \eqref{ine:det:main4} holds. 
	Then we have $\|\vectheta_\eta\|_1\leq 3\sqrt{s}\|\vectheta_\eta\|_2$, that is proved by Proposition  \ref{p:coe-1-2-norm}.
	\subsubsection{Confirmation of \eqref{ine:det:main3}}
	We confirm \eqref{ine:det:main3}.  
	From a similar calculation in Section \ref{subsubsec:main1}, We have
	\begin{align}
		&\sum_{i=1}^n \lambda_o\sqrt{n} \hat{w}_i'\left\{-h(r_{\vecbeta^*+\vectheta_\eta,i}) +h(r_{\vecbeta^*,i})\right\} \tilde{\vecX}_i^\top \vectheta_\eta \nonumber\\
		&\quad \quad = \left|\sum_{i=1}^n  \frac{\lambda_o}{\sqrt{n}} \left\{	-h (\xi_{\lambda_o,i}-\tilde{x}_{\vectheta_\eta,i})+h (\xi_{\lambda_o,i})\right\}  \tilde{\vecx}_i^\top\vectheta_\eta\right|\nonumber \\
		&\quad \quad +\left|\sum_{i\in \mc{O} \cup \left(\mc{I} \cap I_{<}\right)} \frac{\lambda_o}{\sqrt{n}} \left\{	-h (\xi_{\lambda_o,i}-\tilde{x}_{\vectheta_\eta,i})+h (\xi_{\lambda_o,i})\right\}  \tilde{\vecx}_i^\top\vectheta_\eta\right|\nonumber \\
		&\quad \quad +\left|\sum_{i\in \mc{O}} \lambda_o\sqrt{n} \hat{w}_i'\left\{-h(r_{\vecbeta^*+\vectheta_\eta,i}) +h(r_{\vecbeta^*,i})\right\} \tilde{\vecX}_i^\top \vectheta_\eta\right|.
	\end{align}
	Again, we note that $|h(\cdot)|\leq 1$ and from Proposition  \ref{l:w2}, $ |\mc{O} \cup \left(\mc{I} \cap I_{<}\right)|\leq o+2\varepsilon n$., and we remember $\|\vectheta_\eta\|_1\leq 3\sqrt{s}\|\vectheta_\eta\|_2$ holds. Therefore, from Proposition \ref{p:main:out} - \ref{p:main:sc} with $c=2$,
	we see that \eqref{ine:det:main3} holds with 
	\begin{align}
		\label{b}
		b_1 &= \frac{\lambda_{\Sigma}^2}{6},\nonumber \\
		r_{b,1} &= 18\lambda_o^2nr_\delta^2,\nonumber \\
		r_{b,2} &= 24\lambda_o\sqrt{n}\left\{ \sqrt{s}r_{d,\delta}+\left(\sqrt{c_{suc}}\|\Sigma^\frac{1}{2}\|_{\mr{op}}+\sqrt{s\lambda_*} +\sqrt{s\lambda_*'} \right)r_o  + (\sqrt{s\lambda_*'}+\|\Sigma^\frac{1}{2}\|_{\mr{op}} )\sqrt{\varepsilon}+\frac{\sqrt{s}}{\tau_\vecx^2} \right\}.
	\end{align}

	\subsubsection{Confirmation of \eqref{ine:det:main5}}
	From \eqref{a}, \eqref{b} and from the fact that $\lambda_o\sqrt{n}\geq 1$, we see that 
	\begin{align}
	&\frac{r_{b,2} +C_{\lambda_s} +\sqrt{b_1r_{b,1}}}{b_1}\nonumber \\
	&\leq\frac{6}{\lambda_\Sigma^2} \left(r_{b,2}+ r_{a,2} + 3(r_{a,1}+\lambda_s) \sqrt{s}+\sqrt{b_1r_{b,1}}\right)\nonumber \\
	&< \frac{300}{\lambda_\Sigma^2}\times \nonumber \\
	&\quad \left(\lambda_o\sqrt{n} \left[\sqrt{s}\left\{r_{d,\delta}+\frac{\sigma_{\vecx,2}\sigma_{\vecx,4}^2+\sigma_{\vecx,8}^8}{\tau_\vecx^2}+(\sqrt{\lambda_*} + \sqrt{\lambda_*'})r_o+\sqrt{\lambda_*'\varepsilon}\right\}+ \|\Sigma^\frac{1}{2}\|_{\mr{op}}\left( \sqrt{c_{suc}} r_o +\sqrt{\varepsilon}\right)\right]+\frac{\sqrt{s}}{\tau_\vecx^2} + \sqrt{s}\lambda_s\right),
\end{align}
and \eqref{ine:det:main5} holds for a sufficiently large constant $C_s$ such that $C_s\geq 300$, and the proof is complete.


\section{Proofs}
\label{sec:maindet}

\subsection{Proof of Proposition \ref{p:cwpre}}
\begin{proof}
	We note that this proof is almost the same one of Lemma 2 of \cite{FanWanZhu2021Shrinkage}.
	For any $M\in \mathfrak{M}_{r}$, we have
	\begin{align}
		\frac{1}{n} \sum_{i=1}^n  \langle\tilde{\vecx}_i \tilde{\vecx}_i^\top,M\rangle = \underbrace{\frac{1}{n} \sum_{i=1}^n  \langle\tilde{\vecx}_i \tilde{\vecx}_i^\top,M\rangle -\mbb{E}\frac{1}{n} \sum_{i=1}^n   \langle\tilde{\vecx}_i \tilde{\vecx}_i^\top,M\rangle}_{T_1}+ \mbb{E}\frac{1}{n} \sum_{i=1}^n  \langle\tilde{\vecx}_i \tilde{\vecx}_i^\top,M\rangle.
	\end{align}
	First, we evaluate $T_1$.
	We note that, for any $1\leq j_1,j_2\leq d$,
	\begin{align}
		\mbb{E} \tilde{x}_{i_{j_1}}^2\tilde{x}_{i_{j_2}}^2 &\leq \sqrt{\mbb{E}\tilde{x}_{i_{j_1}}^4} \sqrt{\mbb{E}\tilde{x}_{i_{j_2}}^4} \leq \sigma_{\vecx,4}^4,\quad \mbb{E} \tilde{x}_{i_{j_1}}^{2p}\tilde{x}_{i_{j_2}}^{2p}\leq 
		\tau_\vecx^{2(p-2)}\mbb{E} \tilde{x}_{i_{j_1}}^2\tilde{x}_{i_{j_2}}^2 \leq \tau_\vecx^{2(p-2)} \sigma_{\vecx,4}^4.
	\end{align}
	From Bernstein's inequality (Lemma 5.1 of \cite{Dir2015Tail}), we have
	\begin{align}
		 \mbb{P}\left\{\frac{1}{n}\sum_{i=1}^n\tilde{\vecx}_{i_j}\tilde{\vecx}_{i_j}^\top -\mbb{E}\sum_{i=1}^n\frac{1}{n}\tilde{\vecx}_{i_j}\tilde{\vecx}_{i_j}^\top \geq \sigma_{\vecx,4}^2\sqrt{2\frac{t}{n}}+\frac{\tau_{\vecx}^2t}{n}\right\}\leq e^{-t}.
	\end{align}
	From the union bound, we have
	\begin{align}
		\mbb{P}\left\{\left\|\frac{1}{n}\sum_{i=1}^n\tilde{\vecx}_{i_j}\tilde{\vecx}_{i_j}^\top -\frac{1}{n}\sum_{i=1}^n\mbb{E}\tilde{\vecx}_{i_j}\tilde{\vecx}_{i_j}^\top \right\|_\infty \leq \sqrt{2}\sigma_{\vecx,4}^2(r_d+r_\delta)+\tau_{\vecx}^2 (r_d^2+r_\delta^2)\right\}&\geq 1-\delta.
	\end{align}
	From H{\"o}lder's inequality, we have
	\begin{align}
		\mbb{P}\left\{T_1\leq \sqrt{2}\sigma_{\vecx,4}^2(r_d+r_\delta)+\tau_{\vecx}^2 (r_d^2+r_\delta^2)\|M\|_1\right\}&\geq 1-\delta.
	\end{align}
	Next, we evaluate $\mbb{E} \left\langle\tilde{\vecx}_i \tilde{\vecx}_i^\top,M\right\rangle$. 
	We have
	\begin{align}
		\mbb{E} \left\langle\tilde{\vecx}_i \tilde{\vecx}_i^\top,M\right\rangle
		&=\mbb{E} \left\langle\tilde{\vecx}_i \tilde{\vecx}_i^\top-\Sigma,M\right\rangle+\mbb{E} \left\langle \Sigma,M\right\rangle.
	\end{align}
	From  H{\"o}lder's inequality and the positive semi-definiteness of $M$, we have
	\begin{align}
		\mbb{E} \left\langle \Sigma,M\right\rangle\leq \|\Sigma\|_{\mr{op}}\|M\|_*= \|\Sigma\|_{\mr{op}}\mr{Tr}(M).
	\end{align}
	For any $1\leq j_1,j_2\leq d$, we have
	\begin{align}
		\label{ine:k}
		\mbb{E}\tilde{x}_{ij_1}\tilde{x}_{ij_2}-\mbb{E}x_{ij_1}x_{ij_2} &\leq \mbb{E}|x_{ij_1}x_{ij_2} \{\mr{I}_{(x_{ij_1}\geq \tau_\vecx)}+\mr{I}_{(x_{ij_2} \geq \tau_\vecx)}\}|\nonumber\\
		& \leq \sqrt{\mbb{E}{x_{ij_1}^2x_{ij_2}^2}}\left\{\sqrt{\mbb{P} (|x_{ij_1}|\geq \tau_\vecx)}+\sqrt{\mbb{P} (|x_{ij_2}|\geq \tau_\vecx)}\right\}\nonumber\\
		&\leq \sqrt{\mbb{E}{x_{ij_1}^2x_{ij_2}^2}}\left(\sqrt{\mbb{E}\frac{x_{ij_1}^4}{\tau_\vecx^4}}+\sqrt{\mbb{E}\frac{x_{ij_2}^4}{\tau_\vecx^4}}\right)\nonumber \\
		&\leq 2\frac{\sigma_{\vecx,4}^4}{\tau_{\vecx}^2}
	\end{align}
	and from H{\"o}lder's inequality, we have
	\begin{align}
		\mbb{E} \left\langle\tilde{\vecx}_i \tilde{\vecx}_i^\top-\Sigma,M\right\rangle\leq 2\frac{\sigma_{\vecx,4}^4}{\tau_{\vecx}^2}\|M\|_1.
	\end{align}
	Finally, combining the arguments above, with probability at least $1-\delta$, we have
	\begin{align}
		\label{ine:peeling3}
		\left|\sum_{i=1}^n \frac{\left\langle\tilde{\vecx}_i \tilde{\vecx}_i^\top,M\right\rangle}{n}\right|
		&\leq  \left\{\sqrt{2}\sigma_{\vecx,4}^2(r_d+r_\delta)+\tau_{\vecx}^2 (r_d^2+r_\delta^2)+2\frac{\sigma_{\vecx,4}^4}{\tau_{\vecx}^2}\right\}\|M\|_1+\|\Sigma\|_{\mr{op}}r^2
	\end{align}
	and the proof is complete.
\end{proof}

\subsection{Proof of Proposition \ref{p:main}}

We show that $\|\vectheta\|_1\leq 3\sqrt{s}\|\vectheta\|_2$ holds under the assumptions of Proposition \ref{p:main}.
First, we prove the following proposition.
\begin{proposition}
	\label{p:starMRE}
	Consider the output of Algorithm \ref{ourmethod}.
 	Suppose that, for any $\eta\in(0,1)$,
	\begin{align}
		\label{ine:det:xis0}
		\left| \lambda_o\sqrt{n}\sum_{i=1}^n \hat{w}_i'h(r_{\vecbeta^*,i}) \tilde{\vecX}_i^\top \vectheta_\eta \right| \leq  r_{a,2} \|\vectheta_\eta\|_2 + r_{a,1}\|\vectheta_\eta\|_1,
	\end{align}
	where $r_{a,2},r_{a,1} \geq 0$ are some numbers. Suppose that $\lambda_s$ satisfy 
	\begin{align}
		\label{ine:det:par}
		\lambda_s-C_s>0,\quad \frac{\lambda_s+C_s}{\lambda_s-C_s}\leq 2,\text{ where }C_s =r_{a,2} /\sqrt{s} + r_{a,1}.
	\end{align} 
	Suppose that $\|\vectheta_\eta \|_2 \leq \|\vectheta_\eta\|_1/\sqrt{s}$.
Then, we have
	\begin{align}
		\| \vectheta_{\eta,\mc{J}_{\vecbeta^*}^c}\|_1  \leq \frac{\lambda_* + C_s}{\lambda_* - C_s  }\| \vectheta_{\eta,\mc{J}_{\vecbeta^*}}\|_1\left(\leq 2 \| \vectheta_{\eta,\mc{J}_{\vecbeta^*}}\|_1\right),
	\end{align}
	where $\mc{J}_{\veca}$ is the index set of the non-zero entries of $\veca$.
\end{proposition}

\begin{proof}
	Let   
	\begin{align}
	 Q'(\eta) = \lambda_o\sqrt{n} \hat{w}_i'\sum_{i=1}^n \{-h (r_{\vecbeta^*+\vectheta_\eta,i})+h(r_{\vecbeta^*,i}) \}\langle \tilde{\vecX}_i, \vectheta\rangle.
	 \end{align}
	From the proof of Lemma F.2. of \cite{FanLiuSunZha2018Lamm}, we have $\eta Q'(\eta) \leq \eta Q'(1)$ and this means
	\begin{align}
	\label{ine:det:1}
	\sum_{i=1}^n \lambda_o\sqrt{n} \hat{w}_i'\left\{-h (r_{\vecbeta^*+\vectheta_\eta,i})+h(r_{\vecbeta^*,i})\right\} \tilde{\vecX}_i^\top \vectheta_\eta  &\leq \sum_{i=1}^n \lambda_o\sqrt{n} \hat{w}_i'\eta \left\{-h (r_{\hat{\vecbeta},i}) +h(r_{\vecbeta^*,i})\right\}\tilde{\vecX}_i^\top\vectheta.
	\end{align}
	Let $\partial \vecv$ be the sub-differential of $\|\vecv\|_1$.
	Adding $\eta \lambda_s(\|\hat{\vecbeta}\|_1-\|\vecbeta^*\|_1) $ to both sides of \eqref{ine:det:1}, we have
	
	\begin{align}
	\label{ine:det:2}
	&\sum_{i=1}^n \lambda_o\sqrt{n} \hat{w}_i'\left\{-h (r_{\vecbeta^*+\vectheta_\eta,i}) +h(r_{\vecbeta^*,i})\right\}\tilde{\vecX}_i^\top\vectheta_\eta +\eta \lambda_*(\|\hat{\vecbeta}\|_1-\|\vecbeta^*\|_1)\nonumber \\
	 & \stackrel{(a)}{\leq}  \sum_{i=1}^n \lambda_o\sqrt{n} \hat{w}_i'\eta \left\{-h (r_{\hat{\vecbeta},i}) +h(r_{\vecbeta^*,i})\right\}\tilde{\vecX}_i^\top\hat{\vectheta}+\eta \lambda_s\langle \partial \hat{\vecbeta},\vectheta\rangle\stackrel{(b)}{=} \sum_{i=1}^n \lambda_o\sqrt{n} \hat{w}_i'  h(r_{\vecbeta^*,i})\tilde{\vecX}_i^\top\vectheta_\eta,
	\end{align}
	where (a) follows from $\|\hat{\vecbeta}\|_1-\|\vecbeta^*\|_1 \leq \langle \partial \hat{\vecbeta},\vectheta\rangle$, which is the definition of the sub-differential, and (b) follows from the optimality of $\hat{\vecbeta}$.

	From the convexity of the  Huber loss, the left-hand side (L.H.S) of \eqref{ine:det:2} is positive and  we have
	\begin{align}
		\label{ine:det:5}
	0\leq \sum_{i=1}^n \lambda_o\sqrt{n} \hat{w}_i'  h(r_{\vecbeta^*,i})\tilde{\vecX}_i^\top\vectheta_\eta+	\eta \lambda_s(\|\vecbeta^*\|_1-\|\hat{\vecbeta}\|_1).
	\end{align}
	From \eqref{ine:det:xis0}, the first term of the right-hand side (R.H.S.) of \eqref{ine:det:5} is evaluated as
	\begin{align}
		\label{ine:det:6}
		\sum_{i=1}^n \lambda_o\sqrt{n} \hat{w}_i'  h(r_{\vecbeta^*,i})\tilde{\vecX}_i^\top\vectheta_\eta\leq r_{a,2} \| \vectheta_\eta\|_2 + r_{a,1}\|\vectheta_\eta\|_1.
	\end{align}
	From~\eqref{ine:det:5},~\eqref{ine:det:6} and the assumption $\|\vectheta_\eta \|_2 \leq  \|\vectheta_\eta\|_1/\sqrt{s}$, we have
	\begin{align}
	0 &\leq r_{a,2} \| \vectheta_\eta\|_2 + r_{a,1}\|\vectheta_\eta\|_1+\eta \lambda_s(\|\vecbeta^*\|_1-\|\hat{\vecbeta}\|_1)  \leq C_s \|\vectheta_\eta\|_1+\eta \lambda_s (\|\vecbeta^*\|_1-\|\hat{\vecbeta}\|_1).
	\end{align}
	Furthermore, we see
	\begin{align}
	0
	&\leq C_s\|\vectheta_\eta\|_1+\eta \lambda_s(\|\vecbeta^*\|_1-\|\hat{\vecbeta}\|_1) \nonumber \\
	& \leq  C_s (\|\vectheta_{\eta,\mc{J}_{\vecbeta^*}}\|_1+\|\vectheta_{\eta,\mc{J}^c_{\vecbeta^*}}\|_1) + \eta\lambda_s(\|\vecbeta^*_{\mc{J}_{\vecbeta^*}}- \hat{\vecbeta}_{\mc{J}_{\vecbeta^*}}\|_1-\|\hat{\vecbeta}_{\mc{J}^c_{\vecbeta^*}}\|_1)\nonumber\\
	& =\left(\lambda_s +  C_s\right)\|\vectheta_{\eta,\mc{J}_{\vecbeta^*}}\|_1
	+\left(-\lambda_s+  C_s\right)\|\vectheta_{\eta,\mc{J}^c_{\vecbeta^*}}\|_1
	\end{align}
	and the proof is complete.
	\end{proof}

	From Proposition~\ref{p:starMRE}, we can easily prove the following proposition, which reveals a relation between  $\|\vectheta_\eta\|_1$ and $\|\vectheta_\eta\|_2$. 

\begin{proposition}
	\label{p:coe-1-2-norm}
	Suppose the conditions used in Proposition \ref{p:starMRE}.
	Then,  we have
	\begin{align}
		\label{e:l1l2}
		\|\vectheta_\eta\|_1 \leq 3 \sqrt{s}\|\vectheta_\eta\|_2.
	\end{align}
\end{proposition}
\begin{proof}
	When $\|\vectheta_\eta\|_1 < \sqrt{s}\|\vectheta_\eta\|_2$, we obtain \eqref{e:l1l2} immediately.
	When $\|\vectheta_\eta\|_1 \geq \sqrt{s}\|\vectheta_\eta\|_2$,
	from Proposition~\ref{p:starMRE}, we see that $\vectheta_\eta$ satisfies $\|\vectheta_{\eta,\mc{J}^c_{\vecbeta^*}}\|_1\leq 2\| \vectheta_{\eta,\mc{J}_{\vecbeta^*}}\|_1 \leq 2\| \vectheta_{\eta,\mc{J}_{\vecbeta^*}}\|_1$.
	From this, we have
	\begin{align}
		\|\vectheta_\eta\|_1 = \|\vectheta_{\eta,\mc{J}_{\vecbeta^*}}\|_1+ \|\vectheta_{\eta,\mc{J}^c_{\vecbeta^*}}\|_1 \leq (2+1) \|\vectheta_{\eta,\mc{J}_{\vecbeta^*}}\|_1 \leq 3\sqrt{s} \|\vectheta_\eta\|_2,
	\end{align}
	and the proof is complete.
	\end{proof}

\subsubsection{Proving Proposition \ref{p:main}}
\label{sec:mainproof}
In Section \ref{sec:mainproof}, we prove Proposition \ref{p:det:main}.
We note that by combining Propositions \ref{p:starMRE}, \ref{p:coe-1-2-norm} and  \ref{p:det:main}, we see the fact that Proposition \ref{p:main} holds.
\begin{proposition}
	\label{p:det:main}
	Assume all the conditions used in Proposition \ref{p:starMRE}. 
	Suppose that, for any $\eta\in(0,1)$,
	\begin{align}	
		\label{ine:det:lower}
		b_1 \|\vectheta_\eta\|_2^2-r_{b,2}\|\vectheta_\eta\|_2 -r_{b,1}\leq \sum_{i=1}^n \lambda_o\sqrt{n} \hat{w}_i'\left\{-h(r_{\vecbeta^*+\vectheta_\eta,i}) +h(r_{\vecbeta^*,i})\right\} \tilde{\vecX}_i^\top \vectheta_\eta,
	\end{align}
	where  $b_1>0,  r_{b,c},r_{b,1} \geq 0$ are some numbers. Suppose that
	\begin{align}
		\label{ine:det:condc}
		\frac{r_{b,2} +C_{\lambda_s} +\sqrt{b_1r_{b,1}}}{b_1} < r,\text{ where }C_{\lambda_s} = r_{a,2} + 3(r_{a,1}+\lambda_s ) \sqrt{s}.
	\end{align}
	Then, the output of ROBUST-SPARSE-ESTIMATION $\hat{\vecbeta}$ satisfies $\|\hat{\vecbeta} -\vecbeta^*\|_2  \leq r$.
\end{proposition}
\begin{proof}
	We prove Proposition~\ref{p:det:main} in a manner similar to  the proof of Lemma B.7 in \cite{FanLiuSunZha2018Lamm} and the proof of Theorem 2.1 in \cite{CheZho2020Robust}.
	For fixed $r>0$, define
	\begin{align}
		\mbb{B}(r)  :=\left\{ \vecbeta\, :\, \|\vecbeta-\vecbeta^*\|_2 \leq r\right\}.
	\end{align}
	We prove $\hat{\vecbeta} \in \mbb{B}(r) $ by assuming $\hat{\vecbeta} \notin \mbb{B}(r) $ and deriving a contradiction. For $\hat{\vecbeta} \notin \mbb{B}(r) $, we can find some $\eta  \in (0,1)$ such that $\|\vectheta_\eta\|_2=r$.
	From \eqref{ine:det:2}, we have
	\begin{align}
		\label{ine:det2:1}
		&\sum_{i=1}^n \lambda_o\sqrt{n} \hat{w}_i'\left\{-h(r_{\vecbeta^*+\vectheta_\eta,i}) +h(r_{\vecbeta^*,i})\right\} \tilde{\vecX}_i^\top \vectheta_\eta \leq \sum_{i=1}^n \lambda_o\sqrt{n} \hat{w}_i' h(r_{\vecbeta^*,i}) \tilde{\vecX}_i^\top \vectheta_\eta+\eta \lambda_s(\|\vecbeta^*\|_1-\|\hat{\vecbeta}\|_1).
	\end{align}
	We evaluate each term of \eqref{ine:det2:1}. From~\eqref{ine:det:lower}, the L.H.S. of~\eqref{ine:det2:1} is evaluated as
	\begin{align}
		b_1 \|\vectheta_\eta\|_2^2 -r_{b,2}\|\vectheta_\eta\|_2-r_{b,1} \leq\sum_{i=1}^n \lambda_o\sqrt{n} \hat{w}_i'\left\{-h(r_{\vecbeta^*+\vectheta_\eta,i}) +h(r_{\vecbeta^*,i})\right\} \tilde{\vecX}_i^\top \vectheta_\eta.
	\end{align}
	From~\eqref{ine:det:xis0} and~\eqref{e:l1l2} and Proposition \ref{p:coe-1-2-norm}, the first term of the R.H.S. of \eqref{ine:det2:1} is evaluated as
	\begin{align}
		\sum_{i=1}^n \lambda_o\sqrt{n} \hat{w}_i' h(r_{\vecbeta^*,i}) \tilde{\vecX}_i^\top \vectheta_\eta&\leq r_{a,2}\| \vectheta_\eta\|_2 +r_{a,1}\|\vectheta_\eta\|_1\leq \left(
		r_{a,2} +3 \sqrt{s} r_{a,1}\right)\|\vectheta_\eta\|_2.
	\end{align}
	From~\eqref{e:l1l2} and Proposition \ref{p:coe-1-2-norm}, the second term of the R.H.S. of \eqref{ine:det2:1} is evaluated as
	\begin{align}
		\eta \lambda_s(\|\vecbeta^*\|_1-\|\hat{\vecbeta}\|_1) \leq \lambda_s \|\vectheta_\eta\|_1 \leq 3\lambda_s  \sqrt{s} \| \vectheta_\eta\|_2.
	\end{align}
	Combining the  two inequalities  above with \eqref{ine:det2:1}, we have
	\begin{align}
		\label{ine:det:quad}
		b_1 \|\vectheta_\eta\|_2^2 -r_{b,2}\|\vectheta_\eta\|_2-r_{b,1} \leq \{r_{a,2} + 3(r_{a,1}+\lambda_s) \sqrt{s}\}\|\vectheta_\eta\|_2.
 	\end{align}
	From \eqref{ine:det:quad}, $\sqrt{A+B}  \leq \sqrt{A}+\sqrt{B}$ for $A,B>0$, we have
	\begin{align}
		\|\vectheta_\eta\|_2 &\leq \frac{r_{a,2} +r_{b,2} + 3(r_{a,1}+\lambda_s) \sqrt{s}+C_{\lambda_s} +\sqrt{b_1r_{b,1}}}{b_1} < r.
	\end{align}
	This is in contradiction to  $\|\vectheta_\eta\|_2= r$.
	Consequently, we have $\hat{\vecbeta} \in \mbb{B}(r_1)$ and $\|\vectheta\|_2 < r$.
\end{proof}

\subsection{Proofs of the propositions in Section  \ref{sec:keyL}}
\label{sec:p1}

\subsubsection{Proof of proposition \ref{l:w2}}
\begin{proof}
	We assume $|I_{<}| > 2\varepsilon n$, and then we derive a contradiction.
 	From the constraint about $w_i$, we have $0\leq w_i \leq \frac{1}{\left(1-\varepsilon\right)n}$ for any $i \in\{1,\cdots,n\}$ and we have
 	\begin{align}
		\sum_{i=1}^n w_i &= \sum_{i \in I_{<}} w_i+ \sum_{i \in I_{\geq}} w_i \nonumber\\&\leq |I_{_<}| \times \frac{1+}{2n} + (n-|I_{<}|) \times \frac{1}{\left(1-\varepsilon\right)n}\nonumber\\
		& = 2\varepsilon n \times \frac{1}{2n} + (|I_{<}|-2\varepsilon n) \times \frac{1+\varepsilon}{2n} + (n-2\varepsilon n) \times \frac{1}{\left(1-\varepsilon\right)n} +(2\varepsilon n -|I_{<}|) \times \frac{1}{\left(1-\varepsilon\right)n}\nonumber\\
		&= \varepsilon + (n-2\varepsilon n) \times \frac{1}{\left(1-\varepsilon\right)n} +(|I_{<}|-2\varepsilon n) \times\left(\frac{1}{2n}-\frac{1}{\left(1-\varepsilon\right)n}\right)\nonumber\\
		&< \varepsilon + \frac{n-2\varepsilon n}{\left(1-\varepsilon\right)n}\nonumber\\
		&= \varepsilon + \frac{1-2\varepsilon }{1-\varepsilon}\nonumber\\
		&\leq \frac{1-\varepsilon -\varepsilon^2}{1-\varepsilon}\nonumber \\
		&< 1.
	 \end{align}
 This is in contradiction to $	\sum_{i=1}^n w_i=1$. Then, combining the assumption that, we have $|I_{<}| \leq 2\varepsilon n$.
\end{proof}

\subsubsection{Proof of proposition \ref{p:main1}}

\begin{proof}
	This proof is similar to the proof of Lemma 1 of \cite{FanWanZhu2021Shrinkage}.
	For any $\vecv\in \mbb{R}^{d}$, from (iii) of Assumption \ref{a:1}, we have
	\begin{align}
		\label{mcssub2}
		\sum_{i=1}^n  \frac{h(\xi_{\lambda_o,i})\langle \tilde{\vecx}_i,\vecv\rangle}{n}  &=\underbrace{\sum_{i=1}^n \frac{h(\xi_{\lambda_o,i})\langle \tilde{\vecx}_i,\vecv\rangle}{n} -\mbb{E}  h(\xi_{\lambda_o,i}) \langle \tilde{\vecx}_i,\vecv\rangle}_{T_2}\nonumber \\
		&+\underbrace{\mbb{E}   h(\xi_{\lambda_o,i})\langle \tilde{\vecx}_i,\vecv\rangle-\mbb{E}   h\left(\frac{\xi_i}{\lambda_o\sqrt{n}}\right)\langle \tilde{\vecx}_i,\vecv\rangle}_{T_3}\nonumber \\
		&+\underbrace{\mbb{E}   h\left(\frac{\xi_i}{\lambda_o\sqrt{n}}\right)\langle \tilde{\vecx}_i,\vecv\rangle-\mbb{E}   h\left(\frac{\xi_i}{\lambda_o\sqrt{n}}\right)\langle \vecx_i,\vecv\rangle}_{T_4}.
	\end{align}
	First, we evaluate $T_2$.
	We note that
	\begin{align}
		\mbb{E} \left\{h(\xi_{\lambda_o,i})\tilde{\vecx}_{i_j}\right\}^2 &\leq 
		\mbb{E} h(\xi_{\lambda_o,i})^2\tilde{\vecx}_{i_j}^2\leq \mbb{E} \vecx_{i_j}\leq \sigma_{\vecx,2}^2,\nonumber\\
		 \mbb{E} \left\{h(\xi_{\lambda_o,i})\tilde{\vecx}_{i_j}\right\}^p &\leq 
		\mbb{E} h(\xi_{\lambda_o,i})^p\tilde{\vecx}_{i_j}^p\leq \tau_\vecx^{p-2}\sigma_{\vecx,2}^2.
	\end{align}
	From Bernstein's inequality (Lemma 5.1 of \cite{Dir2015Tail}), we have
	\begin{align}
		 \mbb{P}\left\{\frac{1}{n}\sum_{i=1}^nh(\xi_{\lambda_o,i})\tilde{\vecx}_{i_j}-\mbb{E}\sum_{i=1}^n\frac{1}{n}h(\xi_{\lambda_o,i})\tilde{\vecx}_{i_j}\geq \sigma_{\vecx,2}\sqrt{2\frac{t}{n}}+\frac{\tau_{\vecx}t}{n}\right\}\leq e^{-t}.
	\end{align}
	From the union bound, we have
	\begin{align}
		\mbb{P}\left\{\left\|\frac{1}{n}\sum_{i=1}^nh(\xi_{\lambda_o,i})\tilde{\vecx}_i-\mbb{E}h(\xi_{\lambda_o,i})\tilde{\vecx}_{i_j}\right\|_\infty \leq \sqrt{2}\sigma_{\vecx,2}(r_d+r_\delta)+\tau_{\vecx} (r_d^2+r_\delta^2)\right\}&\geq 1-\delta.
	\end{align}
	From H{\"o}lder's inequality, we have
	\begin{align}
		\mbb{P}\left[T_2\leq \left\{\sqrt{2}\sigma_{\vecx,2}(r_d+r_\delta)+\tau_{\vecx} (r_d^2+r_\delta^2)\right\}\|\vecv\|_1\right]&\geq 1-\delta.
	\end{align}
	Second, we evaluate $T_3$.
	\begin{align}
		\label{ine:T_4}
		T_3&= \mbb{E}\langle \tilde{\vecx}_i,\vecv\rangle \left\{h(\xi_{\lambda_o,i})-h\left(\frac{\xi_i}{\lambda_o\sqrt{n}}\right)\right\}\nonumber \\
		&\stackrel{(a)}{=}\mbb{E}\langle \tilde{\vecx}_i,\vecv\rangle \left\{h'\left(t\xi_{\lambda_o,i} +(1-t)\frac{\xi_i}{\lambda_o\sqrt{n}}\right)\right\}\times \frac{(\vecx_{i}-\tilde{\vecx}_i)^\top \vecbeta^*}{\lambda_o\sqrt{n}}\nonumber \\
		&\leq \frac{1}{\lambda_o\sqrt{n}}\sqrt{\mbb{E} |(\vecx_{i}-\tilde{\vecx}_i)^\top \vecbeta^*|^2} \sqrt{ \langle \tilde{\vecx}_i,\vecv\rangle^2},
	\end{align}
	where (a) follows from the mean-valued theorem defining $h'$ as the differential of $h$ and  $t \in (0,1)$.
	We note that, for any $1\leq j_1,j_2\leq d$, we have
	\begin{align}
		\label{ine:k23}
		\mbb{E}(x_{ij_1}-\tilde{x}_{ij_1})(x_{ij_2}-\tilde{x}_{ij_2}) \leq \mbb{E}|x_{ij_1} \{\mr{I}_{|x_{ij_1}| \geq \tau_\vecx}\}||x_{ij_2} \{\mr{I}_{|x_{ij_2}| \geq \tau_\vecx}\}|&\leq (\mbb{E}x^4_{ij_1} x^4_{ij_2})^\frac{1}{4} \mbb{E} \{\mr{I}_{|x_{ij_1}| \geq \tau_\vecx}\})^\frac{3}{4}\nonumber \\
		&\leq \frac{\sigma_{\vecx,8}^8}{\tau_{\vecx}^6},
	\end{align}
	and we have
	\begin{align}
		\label{ine:k24}
		\sqrt{\mbb{E}|(\tilde{\vecx}_{i}-\vecx_{i})^\top \vecbeta^*|^2}\leq \|\vecbeta^*\|_1  \frac{\sigma_{\vecx,8}^4}{\tau_{\vecx}^3}.
	\end{align}
	Additionally, we note that, 
	\begin{align}
		\label{ine:k25}
		\sqrt{ \langle \tilde{\vecx}_i,\vecv\rangle^2} \leq  \sqrt{ \langle \tilde{\vecx}_i-\vecx_i,\vecv\rangle^2}+ \sqrt{ \langle \vecx_i,\vecv\rangle^2} \leq  \sqrt{ \langle \tilde{\vecx}_i-\vecx_i,\vecv\rangle^2} + \|\Sigma^\frac{1}{2}\|_{\mr{op}}\|\vecv\|_2.
	\end{align}
	From \eqref{ine:T_4}, \eqref{ine:k24} and \eqref{ine:k25}, we have
	\begin{align}
		\label{ine:k26}
		T_3 \leq \frac{1}{\lambda_o\sqrt{n}}\|\vecbeta^*\|_1    \frac{\sigma_{\vecx,8}^4}{\tau_{\vecx}^3}\left(\frac{\sigma_{\vecx,8}^4}{\tau_{\vecx}^3}\|\vecv\|_1+\|\Sigma^\frac{1}{2}\|_{\mr{op}}\|\vecv\|_2\right)\stackrel{(a)}{\leq} \frac{\sigma_{\vecx,8}^8}{\tau_{\vecx}^2}\|\vecv\|_1 +\frac{\sqrt{s}}{\tau_{\vecx}^2}\|\vecv\|_2,
	\end{align}
	where (a) follows from the assumption on $\tau_\vecx$.
	Lastly, we evaluate $T_4$. For $1\leq j\leq d$, we have
	\begin{align}
		\label{ine:quade}
		\mbb{E}   h\left(\frac{\xi_i}{\lambda_o\sqrt{n}}\right)(\tilde{\vecx}_{i_j}-\vecx_{i_j})&\leq \mbb{E}   h\left(\frac{\xi_i}{\lambda_o\sqrt{n}}\right)|\vecx_{i_j}| \cdot \mr{I}_{|\vecx_{i_j}|\geq \tau_\vecx}\nonumber \\
		&\leq \sqrt{\mbb{E} \vecx_{i_j}^2 \mbb{E} \mr{I}_{|\vecx_{i_j}|\geq \tau_\vecx}}\nonumber \\
		&= \sqrt{\mbb{E}\vecx_{i_j}^2\mbb{P}(|\vecx_{i_j}|\geq \tau_\vecx)}\nonumber \\
		&\leq \frac{\sigma_{\vecx,2}\sigma_{\vecx,4}^2}{\tau_\vecx^2}.
	\end{align}
	From H{\"o}lder's inequality, we have
	\begin{align}
		T_4\leq \frac{\sigma_{\vecx,2}\sigma_{\vecx,4}^2}{\tau_\vecx^2}\|\vecv\|_1.
	\end{align}

	Combining the arguments above, we have
	\begin{align}
		\left|\sum_{i=1}^n  \frac{h(\xi_{\lambda_o,i})\langle \tilde{\vecx}_i,\vecv\rangle}{n} \right|\leq \|\vecv\|_1\left\{\sqrt{2}\sigma_{\vecx,2}(r_d+r_\delta)+\tau_{\vecx} (r_d^2+r_\delta^2)+\frac{\sigma_{\vecx,2}\sigma_{\vecx,4}^2+\sigma_{\vecx,8}^8}{\tau_\vecx^2}\right\} + \frac{\sqrt{s}}{\tau_{\vecx}^2}\|\vecv\|_2
	\end{align}
	with probability at least $1-\delta$, and the proof is complete.
\end{proof}

Define $\mathfrak{M}_{\vecv,r} = \{M\in \mbb{R}^{d\times d}\,:\, M = \vecv \vecv^\top,\,\|\vecv\|_2 = r\}$.
\subsubsection{Proof of Proposition \ref{p:main:out}}
\begin{proof}
	We note that
	\begin{align}
		\left|\sum_{i \in \mc{O}}\hat{w}'_iu_i \tilde{\vecX}_i^\top\vecv  \right|^2 &\stackrel{(a)}{\leq}   c^2\frac{o}{n}\sum_{i =1}^n\hat{w}'_i |\tilde{\vecX}_i^\top \vecv|^2\stackrel{(b)}{\leq}   2c^2\frac{o}{n}\sum_{i =1}^n\hat{w}_i |\tilde{\vecX}_i^\top \vecv|^2,
	\end{align}
	where (a) follows from H{\"o}lder's inequality and $\|\vecu\|_\infty \leq c$ and $|w_i'|\leq 1/n$, and (b) follows from the fact that $\hat{w}_i'\leq 2\hat{w}_i$ for any $i\in (1,\cdots,n)$.
	We focus on $\sum_{i \in \mc{O}}\hat{w}_i |\tilde{\vecX}_i^\top \vecv|^2$.
	For any $\vecv \in \mbb{R}^d$ such that $\|\vecv\|_2= r$, 
	\begin{align}
		\sum_{i =1}^n\hat{w}_i (\tilde{\vecX}_i^\top \vecv)^2 &= \sum_{i =1}^n\hat{w}_i (\tilde{\vecX}_i^\top \vecv)^2 -\lambda_*\|\vecv\|_1^2+\lambda_*\|\vecv\|_1^2\nonumber\\
		&\stackrel{(a)}{\leq} \sup_{M\in \mathfrak{M}_{r}}\left(\sum_{i =1}^n\hat{w}_i \left\langle \tilde{\vecX}_i\tilde{\vecX}_i^\top, M\right\rangle -\lambda_*\|M\|_1\right)+\lambda_*\|\vecv\|_1^2\nonumber\\
		&\stackrel{(b)}{\leq}  \tau_{suc}+\lambda_*\|\vecv\|_1^2,
	\end{align}
	where (a) follows from the fact that $\mathfrak{M}_{\vecv,r} \subset \mathfrak{M}_{r}$, and (b) follows from \eqref{ine:optM} and $\tau_{suc}'\leq \tau_{suc}$.
	Combining the arguments above and from triangular inequality, we have
	\begin{align}
		\sum_{i \in \mc{O}}\hat{w}'_iu_i \tilde{\vecX}_i^\top\vecv  &\leq \sqrt{2}  cr_o \sqrt{\tau_{suc}} + \sqrt{2}c r_o\sqrt{\lambda_*} \|\vecv\|_1,
	\end{align}
	and the proof is complete.
\end{proof}

\subsubsection{Proof of Proposition \ref{p:main:out2}}
\begin{proof}
	We note that, from H{\"o}lder's inequality, for any $\vecv \in \mbb{R}^d$ such that $\|\vecv\|_2 =r$, we have 
	\begin{align}
		\left|\sum_{i \in I_m}\frac{u_i \tilde{\vecx}_i^\top\vecv }{n} \right|^2  &\leq \sum_{i \in I_m}\frac{1}{n}u_i^2 \sum_{i \in I_m}\frac{1}{n}(\vecx_i^\top\vecv)^2 \leq c^2\frac{m}{n}\sum_{i =1}^n\frac{1}{n} (\vecx_i^\top\vecv)^2.
	\end{align}
	From the fact that $\mathfrak{M}_{\vecv,r} \subset \mathfrak{M}_{r}$, we have
	\begin{align}
		\sum_{i =1}^n\frac{(\vecx_i^\top\vecv)^2 }{n} &= \sum_{i =1}^n\frac{(\vecx_i^\top\vecv)^2 }{n} -\lambda_*'\|\vecv\|_1^2+\lambda_*'\|\vecv\|_1^2\leq \sup_{M\in \mathfrak{M}_{r}}\left(\sum_{i =1}^n\frac{\langle \vecx_i \vecx_i^\top ,M\rangle}{n} -\lambda_*'\|M\|_1\right)+\lambda_*'\|\vecv\|_1^2.
	\end{align}
	From Proposition \ref{p:cwpre} and the definition of $\lambda_*'$, we have
	\begin{align}
		\sup_{M\in \mathfrak{M}_{r}}\left(\sum_{i =1}^n\frac{\langle \vecx_i \vecx_i^\top ,M\rangle}{n} -\lambda_*'\|M\|_1\right)\leq \|\Sigma\|_{\mr{op}}r^2.
	\end{align}
	Combining the arguments above and from triangular inequality, we have
	\begin{align}
		\sum_{i \in I_m}\hat{w}'_iu_i \tilde{\vecx}_i^\top\vecv  \leq c\sqrt{\frac{m}{n}}\|\Sigma^\frac{1}{2}\|_{\mr{op}}r + c\sqrt{\frac{m}{n}}\sqrt{\lambda_*'}\|\vecv\|_1,
	\end{align}
	and the proof is complete.
\end{proof}

\subsubsection{Proof of Proposition \ref{p:main:sc}}
\begin{proof}
	This proposition is proved in a manner similar to  the proof of Proposition B.1 of \cite{CheZho2020Robust}.
 	The L.H.S of \eqref{ine:sc} divided by $\lambda_o^2$ can be expressed as
	\begin{align}
		 \sum_{i=1}^n  \left\{-h(\xi_{\lambda_o,i}-\tilde{x}_{\vecv,i})+h \left(\xi_{\lambda_o,i}\right) \right\}\tilde{x}_{\vecv,i}.
	\end{align}
	From the convexity of the Huber loss,  we have
	\begin{align}
		&\sum_{i=1}^n  \left\{-h(\xi_{\lambda_o,i}-\tilde{x}_{\vecv,i})+h \left(\xi_{\lambda_o,i}\right) \right\}\tilde{x}_{\vecv,i} \geq \sum_{i=1}^n  \left\{-h(\xi_{\lambda_o,i}-\tilde{x}_{\vecv,i})+h \left(\xi_{\lambda_o,i}\right) \right\}\tilde{x}_{\vecv,i}\mr{I}_{E_i},
	\end{align}
	where $\mr{I}_{E_i}$ is the indicator function of the event
	\begin{align}
		E_i :=  ( |\xi_{\lambda_o,i}| \leq 1/2)  \cap (   |\tilde{x}_{\vecv,i}|  \leq  1/2).
	\end{align}
	Define the functions
	\begin{align}
	\label{def:phipsi}
		\varphi(x) =\begin{cases}
		x^2   &  \mbox{ if }  |x| \leq  1/2\\
		(x-1/2)^2   &  \mbox{ if }  1/2\leq x  \leq  1 \\
		(x+1/2)^2   &  \mbox{ if }  -1\leq x  \leq -1/2    \\
		0 & \mbox{ if } |x| >1
	\end{cases} ~\mbox{ and }~
		\psi(x) = I_{(|x| \leq 1/2 ) }.
	\end{align}
	Let $f_i(\vecv) = \varphi(\tilde{x}_{\vecv,i}) \psi(\xi_{\lambda_o,i})$
	and we have
	\begin{align}
	\label{ine:huv-conv-f}
		\sum_{i=1}^n  \left\{-h(\xi_{\lambda_o,i}-\tilde{x}_{\vecv,i})+h \left(\xi_{\lambda_o,i}\right) \right\}\tilde{x}_{\vecv,i} &\geq \sum_{i=1}^n  \tilde{x}_{\vecv,i}^2\mr{I}_{E_i}\stackrel{(a)}{\geq} \sum_{i=1}^n  \varphi(\tilde{x}_{\vecv,i}) \psi(\xi_{\lambda_o,i})=\sum_{i=1}^n f_i(\vecv),
	\end{align}
	where (a) follows from $\varphi(v) \geq v^2$ for $|v| \leq 1/2$.  We note that 
	\begin{align}
	\label{ine:f-1/4}
		f_i(\vecv) \leq\varphi(v_i) \leq \min\left(\tilde{x}_{\vecv,i}^2,1\right).
	\end{align}
	To bound $\sum_{i=1}^n f_i(\vecv)$ from below, for any fixed $ \vecv \in  \mc{R}_{\vecv}$, we have
	\begin{align}
	\label{ine:fbelow}
		\sum_{i=1}^n f_i(\vecv)&\geq \mbb{E}f(\vecv) -\sup_{\vecv' \in \mc{R}_{\vecv}}  \Big|\sum_{i=1}^n f_i(\vecv')-\mbb{E}\sum_{i=1}^n f_i(\vecv')\Big|.
	\end{align}
	Define the supremum of a random process indexed by $\mc{R}_{\vecv}$:
	\begin{align}
	\label{ap:delta}
		\Delta  :=  \sup_{ \vecv' \in \mc{R}_{\vecv}} \left| \sum_{i=1}^n f_i(\vecv') - \mbb{E}\sum_{i=1}^n f_i	(\vecv') \right| .  
	\end{align}
	From \eqref{ine:huv-conv-f} and \eqref{def:phipsi}, we have
	\begin{align}
	\label{ine:aplower:tmp}
		\mbb{E}\sum_{i=1}^n f_i(\vecv)&\geq \sum_{i=1}^n\mbb{E} \tilde{x}_{\vecv,i}^2- \sum_{i=1}^n\mbb{E}\tilde{x}_{\vecv,i}^2 I ( |\tilde{x}_{\vecv,i}| \geq 1/2  ) -  \sum_{i=1}^n\mbb{E}\tilde{x}_{\vecv,i}^2 I\left( |\xi_{\lambda_o,i}|\geq 1/2 \right).
	\end{align}
	We note that, from the definition of $\mc{R}_{\vecv}$, we have
	\begin{align}
		\label{ine:q}
		\mbb{E}(\tilde{\vecx}_i^\top \vecv)^2 = \mbb{E} \vecv^\top(\tilde{\vecx}_i \tilde{\vecx}_i^\top -\vecx_i \vecx_i^\top +\vecx_i \vecx_i^\top )\vecv\geq -9\|\mbb{E}(\tilde{\vecx}_i\tilde{\vecx}_i^\top-\vecx_i \vecx_i^\top)\|_\infty\|\vecv\|_2^2 s+\|\Sigma^\frac{1}{2}\vecv\|_2^2
	\end{align}
	and from \eqref{ine:k} and $\|\Sigma^\frac{1}{2}\vecv\|_2^2\geq \lambda_{\Sigma}^2\|\vecv\|_2^2$, we have
	\begin{align}
		\label{ine:v2}
		-18\frac{\sigma_4^2  s\|\vecv\|^2_2}{\tau^2_\vecx}+ \lambda_{\Sigma}^2\|\vecv\|_2^2\leq \mbb{E}(\tilde{\vecx}_i^\top \vecv)^2.
	\end{align}
	We note that 
	\begin{align}
		\label{ine:v3}
			\frac{\mbb{E}(\tilde{\vecx}_i^\top \vecv)^4}{8} &\leq \frac{\mbb{E}\{(\tilde{\vecx}_i-\vecx_i+\vecx_i)^\top\vecv\}^4}{8},\nonumber \\
			&\leq\mbb{E}\{(\tilde{\vecx}_i-\vecx_i)^\top\vecv\}^4+ \mbb{E}(\vecx_i^\top\vecv)^4\nonumber\\
			&\leq \mbb{E}\{(\tilde{\vecx}_i-\vecx_i)^\top\vecv\}^4+K^4 \{\mbb{E}(\vecx_i^\top \vecv)^2\}^2\nonumber\\
			&\leq \mbb{E}\{(\tilde{\vecx}_i-\vecx_i)^\top\vecv\}^4+ K^4 \|\Sigma^{\frac{1}{2}}\|_{\mr{op}}^4\|\vecv\|_2^4,
	\end{align} 
	and, for any $1\leq j_1,j_2,j_3,j_4\leq d$, we have
	\begin{align}
		\label{ine:k232}
		&\mbb{E}(x_{ij_1}-\tilde{x}_{ij_1})(x_{ij_2}-\tilde{x}_{ij_2})(x_{ij_3}-\tilde{x}_{ij_3})(x_{ij_4}-\tilde{x}_{ij_4})\nonumber  \\
		&\leq \mbb{E}|x_{ij_1} \mr{I}_{|x_{ij_1}| \geq \tau_\vecx}||x_{ij_2} \mr{I}_{|x_{ij_2}| \geq \tau_\vecx}||x_{ij_3} \mr{I}_{|x_{ij_3}| \geq \tau_\vecx}||x_{ij_4} \mr{I}_{|x_{ij_4}| \geq \tau_\vecx}|\\
		&\leq \{\mbb{E}(x_{ij_1} x_{ij_2}x_{ij_3} x_{ij_4})^2\}^\frac{1}{2} \left(\mbb{E} \mr{I}_{|x_{ij_1}| \geq \tau_\vecx}\right)^\frac{1}{2}\nonumber \\
		&\leq \frac{\sigma_{\vecx,8}^8}{\tau_{\vecx}^4}.
	\end{align}
	From \eqref{ine:v3} and \eqref{ine:k232}, we have
	\begin{align}
		\label{ine:v4}
			\mbb{E}(\tilde{\vecx}_i^\top \vecv)^4 \leq 8\left\{81s^2\frac{\sigma_{\vecx,8}^8}{\tau_{\vecx}^4}+ K^4 \|\Sigma^{\frac{1}{2}}\|_{\mr{op}}^4\right\}\|\vecv\|_2^4\leq 16 K^4 \|\Sigma^{\frac{1}{2}}\|_{\mr{op}}^4\|\vecv\|_2^4,
	\end{align}
	We evaluate the right-hand side of \eqref{ine:aplower:tmp} at each term.
	First, we have
	\begin{align}
	\label{ap:ine:cov1}
		\sum_{i=1}^n\mbb{E}\tilde{x}_{\vecv,i}^2 I ( |\tilde{x}_{\vecv,i}| \geq 1/2 ) 
		&\stackrel{(a)}{\leq} 	\sum_{i=1}^n\sqrt{\mbb{E} \tilde{x}_{\vecv,i}^4 } \sqrt{\mbb{E}   \ I ( |\tilde{x}_{\vecv,i}| \geq 1/2  ) }\nonumber \\
		&\stackrel{(b)}{\leq}	\sum_{i=1}^n4\mbb{E} \tilde{x}_{\vecv,i}^4\nonumber\\
		&=\frac{4}{\lambda_o^4n}\mbb{E} \langle \tilde{x}_i,\vecv\rangle^4\nonumber\\
		&\stackrel{(c)}{\leq }\frac{64}{\lambda_o^4n} K^4 \|\Sigma^{\frac{1}{2}}\|_{\mr{op}}^4\|\vecv\|_2^4\stackrel{(d)}{\leq}\frac{\lambda_{\Sigma}^2}{3\lambda_o^2}  \|\vecv\|_2^2 ,
	\end{align}
		where (a) follows from H{\"o}lder's inequality, (b) follows from the relation between indicator function and expectation and Markov's inequality, and  (c) follows  from \eqref{ine:v4}, and (d) follows from the definition of $\lambda_o$ and $\|\vecv\|_2\leq 1$.
		Second,  we have
	\begin{align}
	\label{ap:ine:cov2}
		\sum_{i=1}^n\mbb{E} \tilde{x}_{\vecv,i}^4 I( \left|\xi_{\lambda_o,i}\right|\geq 1/2) 
		&\stackrel{(a)}{\leq} \sum_{i=1}^n\sqrt{\mbb{E} \tilde{x}_{\vecv,i}^4}  \sqrt{\mbb{E}I( \left|\xi_{\lambda_o,i} \right|\geq 1/2)}\nonumber \\
		&\stackrel{(b)}{\leq}\sum_{i=1}^n \sqrt{\frac{2}{\lambda_o\sqrt{n}}}\sqrt{\mbb{E}\tilde{x}_{\vecv,i}^4}  \sqrt{\mbb{E}|\xi_i|+\mbb{E}|(\tilde{\vecx}_i-\vecx)^\top \vecbeta^*|}\nonumber \\
		&\stackrel{(c)}{\leq}\sum_{i=1}^n \sqrt{\frac{2}{\lambda_o\sqrt{n}}}\sqrt{\mbb{E}\tilde{x}_{\vecv,i}^4}  \sqrt{\sigma+\sqrt{\mbb{E}|(\tilde{\vecx}_i-\vecx)^\top \vecbeta^*|^2}}\nonumber \\
		&\stackrel{(d)}{\leq}\sum_{i=1}^n \sqrt{\frac{2}{\lambda_o\sqrt{n}}}\sqrt{\mbb{E}\tilde{x}_{\vecv,i}^4}  \sqrt{\sigma+\|\vecbeta^*\|_1  \frac{\sigma_{\vecx,8}^4}{\tau_{\vecx}^3}}\nonumber \\
		&=\frac{1}{\lambda_o^2} \sqrt{\frac{2}{\lambda_o\sqrt{n}}}\sqrt{\mbb{E}\langle\tilde{x}_i,\vecv\rangle^4}  \sqrt{\sigma+\|\vecbeta^*\|_1  \frac{\sigma_{\vecx,8}^4}{\tau_{\vecx}^3}}\nonumber \\
		&\stackrel{(e)}{\leq}\frac{4K\|\Sigma^{\frac{1}{2}}\|_{\mr{op}}^2}{\lambda_o^2} \sqrt{\frac{2}{\lambda_o\sqrt{n}}}  \sqrt{\sigma+\|\vecbeta^*\|_1  \frac{\sigma_{\vecx,8}^4}{\tau_{\vecx}^3}}\|\vecv\|_2^2\nonumber \\
		&\stackrel{(f)}{\leq}\frac{4\sqrt{2}K\|\Sigma^{\frac{1}{2}}\|_{\mr{op}}^2}{\lambda_o^2} \sqrt{\frac{\sigma+1}{\lambda_o\sqrt{n}}}  \|\vecv\|_2^2\stackrel{(g)}{\leq} \frac{\lambda_{\Sigma}^2}{3\lambda_o^2} 	\|\vecv\|_2^2,
	\end{align}
	where (a) follows from H{\"o}lder's inequality, (b) follows from relation between indicator function and expectation and  Markov's inequality, and (c) follows from the assumption on $\{\xi_i\}_{i=1}^n$ and H{\"o}lder's inequality, (d) follows from \eqref{ine:k24}, (e) follows from from \eqref{ine:v4}, (f) follows from the assumption on $\tau_{\vecx}$, and (g) follows from  the definition of $\lambda_o$.
	Consequently, from \eqref{ine:huv-conv-f}, \eqref{ine:fbelow}, \eqref{ine:v2}, \eqref{ap:ine:cov1} and \eqref{ap:ine:cov2}, we have
	\begin{align}
	\label{ap:h_bellow}
		\frac{\lambda_{\Sigma}^2}{6\lambda_o^2} \|\vecv\|_2^2-\Delta\leq 	\sum_{i=1}^n  \left\{-h(\xi_{\lambda_o,i}-\tilde{x}_{\vecv,i})  +h \left(\xi_{\lambda_o,i}\right) \right\}\tilde{x}_{\vecv,i} ,
	\end{align}
	where we use the assumption $9\sigma_4^2 s/\tau^2_\vecx \leq \lambda_{\Sigma}^2/12$.
	Next we evaluate the stochastic term $\Delta$ defined in \eqref{ap:delta}. 
	From \eqref{ine:f-1/4} and Theorem 3 of \cite{Mas2000Constants}, with probability at least $1-\delta$, we have
	\begin{align}
	\label{ine:delta}
		\Delta & \leq  2 \mbb{E} \Delta + \sigma_f \sqrt{8\log(1/\delta)} + 18\log(1/\delta),
	\end{align}
	where $\sigma^2_f= \sup_{ \vecv \in \mc{R}_{r}} \sum_{i=1}^n\mbb{E}  \{f_i(\vecv)-\mbb{E}f_i(\vecv)\}^2$.
	From  \eqref{ine:f-1/4}, $r \leq 1$, \eqref{ine:v4} and the definition of $\lambda_o$, we have
	\begin{align}
	\label{ap:ine:cov3}
	\mbb{E}\{f_i(\vecv)-\mbb{E}f_i(\vecv)\}^2 \leq \mbb{E}f_i^2(\vecv) \leq \mbb{E}\tilde{x}_{\vecv,i}^4\leq \frac{1}{\lambda_o^2n} \|\vecv\|_2^2.
	\end{align}
	Combining this and \eqref{ine:delta}, we have
	\begin{align}
	\label{ap:delta_upper}
		\Delta &\leq 2 \mbb{E} \Delta+ 4\frac{\sqrt{\log(1/\delta)}}{\lambda_o\sqrt{n}}\|\vecv\|_2+ 18\log(1/\delta)
	\end{align}
	with probability at least $1-\delta$.
	From symmetrization inequality (Lemma 11.4 of \cite{BouLugMas2013concentration}), we have  $\mbb{E}\Delta \leq 2   \,\mbb{E} \sup_{ \mc{R}_{\vecv}} |  \mathbb{G}_{\vecv} |  $,
	where 
	\begin{align}
		\mbb{G}_{\Theta} := \sum_{i=1}^n 
		a_i \varphi (\tilde{x}_{\vecv,i}) \psi (\xi_{\lambda_o,i}),
	\end{align} 
	and $\{a_i\}_{i=1}^n$ is a sequence of i.i.d. Rademacher random variables which are independent of $\{\tilde{\vecx}_i,\xi_i\}_{i=1}^n$.
	We  denote $\mbb{E}^*$ as a conditional variance of $\left\{a_i\right\}_{i=1}^n$ given $\left\{\tilde{\vecx}_i,\xi_i\right\}_{i=1}^n$. From contraction principle (Theorem 11.5 of \cite{BouLugMas2013concentration}),  we  have
	\begin{align}
		&\mbb{E} ^*\sup_{\vecv\in\mc{R}_{\vecv}} \left| \sum_{i=1}^n a_i \varphi (\tilde{x}_{\vecv,i}) \psi (\xi_{\lambda_o,i})    \right| \leq	\mbb{E}^* \sup_{\vecv\in\mc{R}_{\vecv}} \left|  \sum_{i=1}^n   a_i \varphi (\tilde{x}_{\vecv,i}) \right|
	\end{align}	
	and from the basic property of the expectation, we have
	\begin{align}
		\mbb{E}\sup_{\vecv\in\mc{R}_{\vecv}}  \left|     \sum_{i=1}^n a_i \varphi (\tilde{x}_{\vecv,i}) \psi (\xi_{\lambda_o,i} )    \right| &\leq	
		\mbb{E}\sup_{\vecv\in\mc{R}_{\vecv}}  \left|  \sum_{i=1}^n   a_i \varphi(\tilde{x}_{\vecv,i}) \right|.
	\end{align}	
	Since $\varphi$ is $\frac{1}{2}$-Lipschitz and $\varphi(0)=0$,  from contraction principle (Theorem 11.6 in \cite{BouLugMas2013concentration}), we have
	\begin{align}
		\mbb{E} \sup_{\vecv\in\mc{R}_{\vecv}} \left|\sum_{i=1}^n a_i \varphi(\tilde{x}_{\vectheta,i})\right|&\leq\frac{1}{\lambda_o\sqrt{n}}	\mbb{E}\sup_{\vecv\in\mc{R}_{\vecv}}  \left|  \sum_{i=1}^n   a_i \tilde{\vecx}_i^\top \vecv \right|.
	\end{align}
	From Proposition \ref{p:1e} and H{\"o}lder's inequality, we have
	\begin{align}
	\label{ine:hub-stoc-upper}
	\frac{1}{n}\mbb{E}\sup_{\vecv\in\mc{R}_{r}}  \left|     \sum_{i=1}^n a_i \varphi (\tilde{x}_{\vecv,i}) \psi (\xi_{\lambda_o,i} )    \right| \leq \|\vecv\|_1(2\sigma_{\vecx,2}r_d+4\tau_\vecx r_d^2)\leq 12\sqrt{s}(\sigma_{\vecx,2}r_d+\tau_\vecx r_d^2)\|\vecv\|_2.
	\end{align}
	Combining \eqref{ine:hub-stoc-upper}, \eqref{ap:delta_upper}, $s\geq 0$ and the definition of $r_{d,\delta}$, we have
	\begin{align}
		\label{ap:delta_upper2}
			\lambda_o^2\Delta &\leq 24\lambda_o\sqrt{n} \sqrt{s}(\sigma_{\vecx,2}r_d+\tau_\vecx r_d^2)\|\vecv\|_2+ 4\lambda_o\sqrt{n}r_\delta\|\vecv\|_2+  18\lambda_o^2nr_\delta^2\nonumber \\
			&\leq 24\lambda_o\sqrt{n} \sqrt{s}r_{d,\delta}\|\vecv\|_2+  18\lambda_o^2nr_\delta^2
		\end{align}
		and from \eqref{ap:h_bellow}, the proof is complete.
\end{proof}

The following proposition is used in the proof of Proposition \ref{p:main:sc}.
\begin{proposition}
	\label{p:1e}
	Suppose that Assumption \ref{a:1} holds.
	Then, we have
	\begin{align}
		\label{ine:gc-normale}
		&\mbb{E}\sup_{\vecv \in \mc{R}_{\vecv}}\left|\frac{1}{n}\sum_{i=1}^n a_i \tilde{\vecx}_i^\top \vecv\right|  \leq 12 \sqrt{s} (\sigma_{\vecx,2}r_d+\tau_\vecx r_d^2)\|\vecv\|_2.
	\end{align}
\end{proposition}
\begin{proof}
	From H{\"o}lder's inequality, we have
	\begin{align}
		\mbb{E}\sup_{\vecv \in \mc{R}_{\vecv}} \left|\frac{1}{n}\sum_{i=1}^n a_i \tilde{\vecx}_i^\top \vecv \right|\leq \mbb{E}\|\vecv\|_1\left\|\frac{1}{n}\sum_{i=1}^n a_i \tilde{\vecx}_i\right\|_\infty \leq 3\sqrt{s} \|\vecv\|_2 \mbb{E}\left\|\frac{1}{n}\sum_{i=1}^n a_i \tilde{\vecx}_i\right\|_\infty.
	\end{align}
	We note that, 
	\begin{align}
		\mbb{E} a_i^2 \tilde{x}_{i_{j}}^2 &\leq \mbb{E}\tilde{x}_{i_j}^2  \leq \sigma_{\vecx,2}^2,\quad \mbb{E} a_i^p \tilde{x}_{i_j}^p\leq 
		\tau_\vecx^{p-2}\mbb{E} \tilde{x}_{i_j}^2 \leq \tau_\vecx^{p-2} \sigma_{\vecx,2}^2.
	\end{align}
  From Lemma 14.12 of \cite{BulGee2011Statistics} and $d\geq 3$, we have
	\begin{align}
		 \mbb{E}\left\|\sum_{i=1}^n a_i\frac{\tilde{\vecx}_i }{n}\right\|_\infty \leq \sqrt{2\frac{\sigma_{\vecx,2}^2\log(d+1)}{n}}+2\tau_\vecx\frac{\log(d+1)}{n}\leq 2\sigma_{\vecx,2}r_d+4\tau_\vecx r_d^2.
	\end{align}
	Combining the arguments above, the proof is complete.
\end{proof}

\bibliographystyle{plain}
\bibliography{RSELRCWHTC}
\end{document}